\documentclass{article}
\PassOptionsToPackage{sort}{natbib}
\usepackage{preprint,times}

\usepackage{amsmath,amsfonts,bm}

\def\eqref#1{equation~\ref{#1}}
\def\1{\bm{1}}

\DeclareMathAlphabet{\mathsfit}{\encodingdefault}{\sfdefault}{m}{sl}
\SetMathAlphabet{\mathsfit}{bold}{\encodingdefault}{\sfdefault}{bx}{n}

\newcommand{\E}{\mathbb{E}}

\newcommand{\KL}{D_{\mathrm{KL}}}

\usepackage{hyperref}
\usepackage{url}
\definecolor{mydarkblue}{rgb}{0,0.08,0.45}
\hypersetup{colorlinks=true,linkcolor=mydarkblue,citecolor=mydarkblue,urlcolor=mydarkblue}
\usepackage{xcolor}
\usepackage{amsthm}
\usepackage{bm}
\usepackage{booktabs}
\usepackage[sort]{natbib}
\usepackage[capitalise,nameinlink]{cleveref}
\usepackage{enumitem}
\usepackage{graphicx}
\usepackage{wrapfig}
\usepackage{subcaption}

\newcommand{\btheta}{\bm{\theta}}
\newcommand{\Dn}{D_n}
\newcommand{\dn}{d_n}
\newcommand{\DN}{D_N}
\newcommand{\prior}{\Pi}
\newcommand{\sizeprior}{\Pi_N}
\newcommand{\ppd}[1]{\pi(#1)}
\newcommand{\ppdbar}[1]{\bar{\pi}(#1)}
\newcommand{\pfn}[1]{q_{\btheta}(#1)}
\newcommand{\task}[1]{p(#1)}
\newcommand{\gap}{\Delta_{\btheta}}

\newcommand{\Ycal}{\mathcal{Y}}
\newcommand{\Xcal}{\mathcal{X}}
\newcommand{\featurecollection}{\mathfrak{X}}
\newcommand{\featurecollectionandD}{\mathfrak{D}}
\newcommand{\responsecollection}{\mathfrak{Y}}
\newcommand{\featurevalues}{\mathfrak{x}}
\newcommand{\responsevalues}{\mathfrak{y}}

\newcommand{\predgap}{\Delta_\pi(y\mid \bm{x},\dn)}
\newcommand{\predgapnoarg}{\Delta_\pi}

\theoremstyle{definition}
\newtheorem{definition}{Definition}
\theoremstyle{plain}
\newtheorem{proposition}{Proposition}

\newtheorem{corollary}{Corollary}
\crefname{proposition}{Proposition}{Propositions}
\crefname{corollary}{Corollary}{Corollaries}
\crefname{definition}{Definition}{Definitions}
\crefname{remark}{Remark}{Remarks}

\title{Characterizing the Separability Gap in Prior-Data Fitted Networks}

\author{Julius Kobialka, Manuel H\"ulskamp, Thomas Nagler, David R\"ugamer\\
Department of Statistics, LMU Munich\\
Munich Center for Machine Learning\\
\texttt{first.lastname@stat.uni-muenchen.de}
}

\iclrfinalcopy
\begin{document}

\maketitle

\begin{abstract}
Prior-data fitted networks (PFNs) are transformers pretrained on simulated datasets to approximate Bayesian posterior predictive distributions. Standard PFN theory assumes a separable task prior, under which the feature distribution and the mechanism generating labels from features are independent. We show that common PFN priors are non-separable, even in simple settings. Consequently, standard pretraining targets a predictive distribution based on a task posterior updated using both the labeled context and the query features. We characterize its difference from the predictive based on the context posterior alone and show that this difference vanishes under posterior concentration. Under misspecification (i.e., when the true joint distribution lies outside the prior's support), however, prior non-separability can favor better feature fit at the expense of conditional label fit. We illustrate the consequences for feature attribution, Bayesian optimization, and martingale checks of Bayesian behavior. The same dependence also enables PFNs to learn from additional unlabeled observations. We demonstrate this with SemiTabICL, a PFN explicitly pretrained to use these observations, which improves accuracy over a matched supervised baseline. Our findings highlight the need for explicit rules on unlabeled data use in PFN evaluation and establish the dependence between feature and label distributions as a consequential prior design choice.
\end{abstract}

\section{Introduction}
\label{sec:intro}

Tabular foundation models such as TabPFN and TabICL are prior-data fitted networks~\citep[PFNs;][]{qu2025tabicl,hollmann2025tabpfnv2,hollmanntabpfn,grinsztajn2025tabpfnv25,grinsztajn2026tabpfnv3,qu2026tabiclv2}. A PFN is an approximator to a family of posterior predictive distributions \citep[PPDs;][]{nagler23statistical}, usually implemented as a transformer with parameters $\btheta$ and pretrained on simulated datasets drawn from a prior $\prior$ over tasks. A common data prior choice are structural causal models (SCMs).
We use uppercase letters for random data and lowercase letters for their realizations. At inference, the PFN receives an observed context $\dn = \{y_i, \bm{x}_i\}_{i=1}^n$ and a query feature vector $\bm{x}$ and returns a predictive distribution $\pfn{y \mid \bm{x}, \dn}$ in its forward pass. 
For classification, PFNs are pretrained by cross-entropy on simulated contexts and query pairs drawn from the same task. The target is therefore the conditional distribution of the query label given its feature vector and the context~\citep[Theorem 2.2]{nagler23statistical}.
A trained PFN approximates this target with finite capacity and a finite pretraining budget.

While previous research on PFNs has asked \textit{how accurately} a trained network approximates its target PPD, we ask a prior question:
\textit{Which predictive distribution does pretraining actually target?} As outlined below, our motivation stems from the (often implicitly made) assumption of \textit{separability} of the data prior, which has so far received too little attention. 

\textbf{Bayesian target}\quad Concretely, let $p \sim \prior$ denote a task distribution, with conditional label distribution $\task{y \mid \bm{x}}$. Here, we distinguish between two types of posterior predictives:
\begin{equation}
\begin{aligned}
    & \text{\itshape \,\,\,\,query-updated predictive:}  & \ppd{y \mid \bm{x}, \dn}
    &= \int \task{y \mid \bm{x}}\, d\prior(p \mid \dn, \bm{x}), \\
    & \text{\itshape context-updated predictive:} & \ppdbar{y \mid \bm{x}, \dn}
    &= \int \task{y \mid \bm{x}}\, d\prior(p \mid \dn).
\end{aligned}
\label{eq:ppd}
\end{equation}
The \emph{query-updated predictive} $\pi$ is the PPD that minimizes pretraining cross-entropy. The \emph{context-updated predictive} $\bar\pi$ averages the same label distribution under the task posterior given the context alone. Both predictives condition on the context features and labels, and evaluate the label distribution at $\bm{x}$. However, only $\pi$ uses $\bm{x}$ to update the task posterior. 

Crucially, the \textbf{existing literature assumes that PFN pretraining targets the context-updated predictive} $\bar\pi$. 
\citet{nagler23statistical} notes in Remark 2.1 that this requires the feature
distribution and the labeling mechanism to be independent under the prior, a
property we call \emph{separability}, but regards the query-dependent alternative
$\pi$ as unintuitive.
This distinction especially matters for the SCM priors used to pretrain state-of-the-art PFNs, as they are typically not separable.

The setting of \Cref{fig:toy_query_evidence} exemplifies this difference using the anticausal setting from \citet{schoelkopf2012causal}. The task is a pair of class centers $\mu_0, \mu_1$ drawn from $\mathcal{N}(0, \tau^2 \bm{I}_2)$, $Y$ is determined by a fair coin flip, and the features are drawn from that label's Gaussian, $\bm{X} \mid Y=k \sim \mathcal{N}(\mu_k, \sigma^2 I_2)$. If the context contains many points of class~$0$ but only a few of class~$1$, the posterior pins down $\mu_0$ (blue cluster) but leaves $\mu_1$ (red cluster) comparatively uncertain. For a query $\bm{x}$ far from both clusters (white star), the query-updated predictive $\pi$ is more likely to assign it label~$1$ due to the smaller probability of $\bm{x}$ belonging to features of class $0$, while the context-updated predictive $\bar\pi$ does not take the feature likelihood into account.

Most current PFNs are trained on non-separable priors like in the previous example. This therefore not only demonstrates that the two models in \cref{eq:ppd} lead to different predictive distributions, but also shows that PFNs have a close link to semi-supervised machine learning as they leverage the unlabeled query features. In this paper, we outline how this perspective proves useful for unlocking empirical performance gains as well as obtaining a better understanding of the working mechanisms of PFNs. 

\begin{figure}
    \centering
    \includegraphics[width=0.73\linewidth]{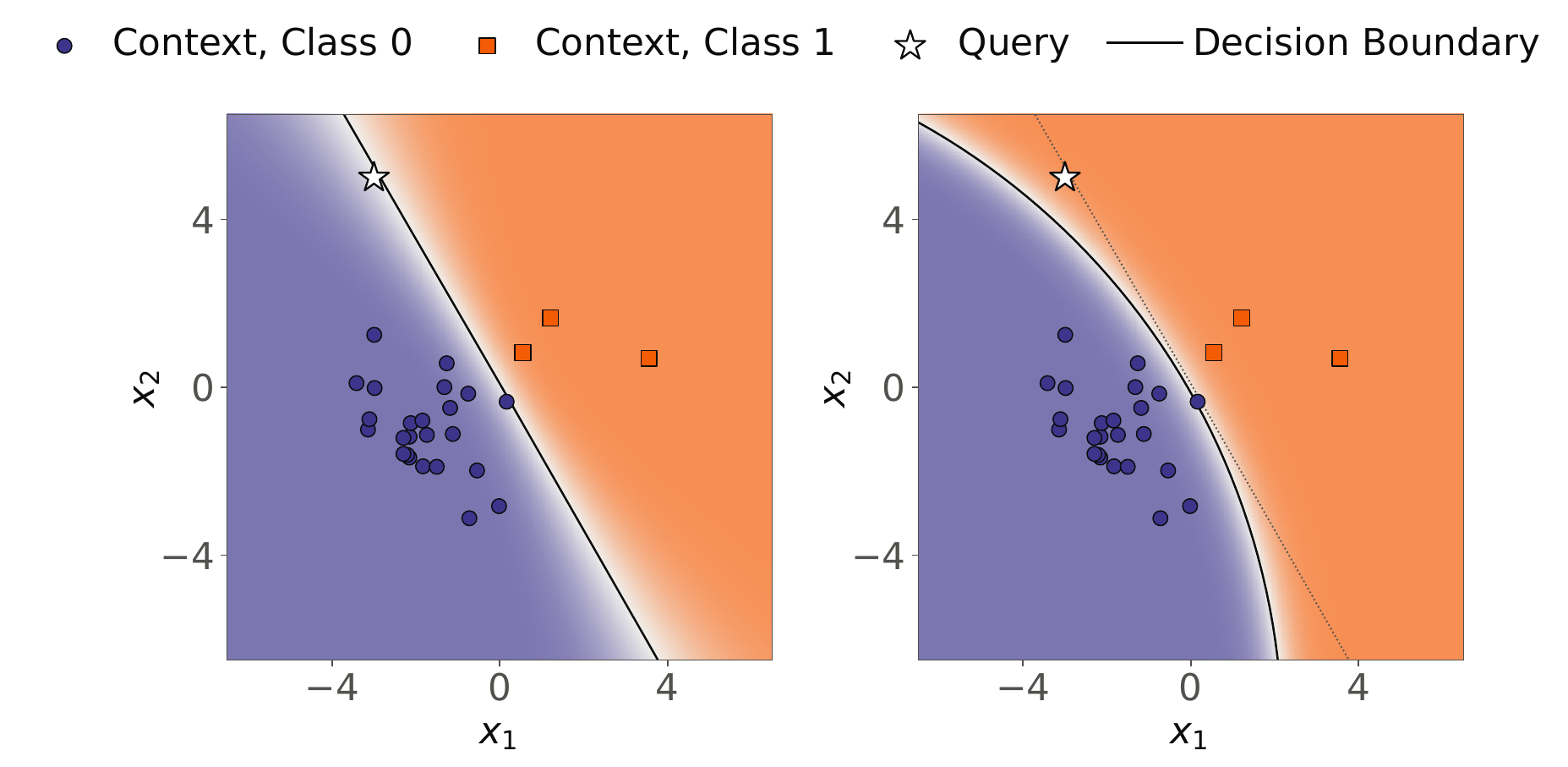}
    \caption{Illustration of the anticausal setting from \citet{schoelkopf2012causal}: Two Gaussian class centers
$\mu_0, \mu_1$ generate the features from the label, such that the location of an unlabeled point is informative about the labeling mechanism. The mechanism in the left figure does not allow
the query's location to update the belief about the centers; the right figure does use the point's position.}
    \label{fig:toy_query_evidence}
\end{figure}

\textbf{Our contributions}\quad By studying the difference of predictives in \cref{eq:ppd} through various lenses, we make the following contributions:
\begin{itemize}[leftmargin=1.5em]
    \item We prove that prior non-separability already arises in simple SCMs and that mixing (separable) priors, as in TabPFN, TabICL, and Mitra, preserves separability if and only if the regimes share the same prior over feature distributions or the same prior over conditional label distributions.
    \item We characterize a \emph{predictive gap} between $\pi$ and $\bar\pi$ at finite context sizes and show that it vanishes for typical queries under posterior concentration. Under misspecification, their common limit can nevertheless have an arbitrarily large \emph{optimality gap} relative to the best supported conditional label law.
    \item We study the practical consequences for downstream use and evaluation of PFNs, including interpretability, Bayesian optimization, semi-supervised prediction, and martingale checks. We test which type of predictive standard pretraining recovers and whether additional unlabeled rows improve prediction.
\end{itemize}

\section{Properties of the different predictives}
\label{sec:properties}

\paragraph{Model and notation}
We identify a task with its distribution $p$, comprising a feature law $p_X$ and a conditional label law $p_{Y\mid X}$. We use lowercase $p$ for both a probability law and its density or mass function, with the meaning determined by the argument; for example, $p(y\mid\bm{x})=p_{Y\mid X}(y\mid\bm{x})$ and $p(\bm{x})=p_X(\bm{x})$. The prior $\prior$ is a distribution over these task laws, so posterior integrals are taken directly over $p$.
Given a task $p \sim \Pi$, let $\bm{X}_1,\dots,\bm{X}_n$ denote the context features and $\bm{X}$ the query feature. Given a task $p\sim\prior$, the context rows $(Y_1,\bm{X}_1),\dots,(Y_n,\bm{X}_n)$ and the query row $(Y,\bm{X})$ are independent and identically distributed (iid) draws from $p$. Each feature vector is drawn from $p_X$ and its label from $p_{Y\mid X}$, that is, $\bm{X}_i\sim p_X$ and $Y_i\mid\bm{X}_i\sim p(\cdot\mid\bm{X}_i)$, and likewise for $(Y,\bm{X})$. The random context is $\Dn=(Y_i,\bm{X}_i)_{i=1}^n$, its realization is $\dn=(y_i,\bm{x}_i)_{i=1}^n$, and the predictor receives $(\bm{X},\Dn)$.

The following definition is central in this work.
\begin{definition}[Separable prior]
\label{def:separable}
Identify each task $p=p_{Y\mid X}p_X$ with the pair $(p_X,p_{Y\mid X})$. The prior is \emph{separable} if its induced law on these pairs factorizes as
\[
    \prior = \prior_X\otimes\prior_{Y\mid X},
    \qquad
    d\prior(p)
    = d\prior_X(p_X)\,d\prior_{Y\mid X}(p_{Y\mid X}).
\]
Here $\prior_X$ and $\prior_{Y\mid X}$ are the marginal priors over the feature law and conditional label law. Thus these two laws are independent under $\prior$. Otherwise the prior is \emph{non-separable}.
\end{definition}

Under separability, the marginal posterior of $p_{Y\mid X}$ does not depend on the query feature $\bm{x}$:
\begin{align*}
    d\prior(p_{Y\mid X}\mid\dn,\bm{x})
    &= d\prior(p_{Y\mid X}\mid\dn) \propto \prod_{i=1}^n p_{Y\mid X}(y_i\mid\bm{x}_i)\, d\prior_{Y\mid X}(p_{Y\mid X});
\end{align*}
see \cref{prop:nonseparable-minimizer} in \cref{app:bop}.
Separability is thus sufficient for equality of the two predictives in \cref{eq:ppd}, i.e.~$\pi=\bar\pi$. As shown in the following, a non-separable prior permits $\pi\neq\bar\pi$, but does not require it at every context and query. Under a non-separable prior, the context features are also informative about $p_{Y\mid X}$, and $\bar\pi$ uses this information because the posterior $\prior(\cdot\mid\dn)$ contains their likelihood $\prod_{i=1}^n p_X(\bm{x}_i)$. The two predictives therefore differ only in the feature likelihood $p_X(\bm{x})$ of the query.

\subsection{SCM priors are non-separable}

In an SCM prior, a random graph, random node functions, and a random choice of the response node generate the features and the label jointly. The feature distribution and the conditional label distribution are then two functionals of the same random mechanism, and thus we should not expect separability to hold for these priors. In fact, it already fails in simple cases like a linear Gaussian SCM with a random label node (see \cref{prop:scm-nonseparable}).

Mixing multiple generator regimes provides another potential source of non-separability. The original TabPFN prior mixes SCM and Bayesian neural network datasets equally \citep{hollmanntabpfn}. TabICL mixes neural and tree-based SCM datasets in a $70/30$ ratio \citep{qu2025tabicl}. Mitra mixes SCM and tree-based priors \citep{zhang2025mitra}, and Mitra-v2 adds a hybrid SCM to this outer mixture \citep{tao2026mitrav2}. TabICLv2 also combines neural, tree, Gaussian-process, and other functions within individual SCMs \citep{qu2026tabiclv2}. Even if the feature and label laws of these priors were separable within each regime, their mixture is non-separable when both induced laws differ across regimes, as the following result shows (proof in \cref{{proof:mixture-nonseparability}}).

\begin{proposition}
\label{prop:mixture-nonseparability}
Let $\prior_{X,s}$ and $\prior_{Y\mid X,s}$ be priors on the induced feature and conditional label laws for $s \in \{0, 1\}$. For $\lambda \in (0, 1)$, set
\begin{align*}
    \prior
    &= (1 - \lambda)\,\prior_{X,0} \otimes \prior_{Y\mid X,0}
       + \lambda\,\prior_{X,1} \otimes \prior_{Y\mid X,1}.
\end{align*}
Then $\prior$ is separable if and only if $\prior_{X,0} = \prior_{X,1}$ or $\prior_{Y\mid X,0} = \prior_{Y\mid X,1}$.
\end{proposition}

A mixture of regimes that differ in both their feature and their label distributions, as in all of the production priors above, is therefore non-separable, whatever the structure within each regime.

\subsection{Finite-sample difference}

So far, we have shown that PFN pretraining targets $\pi$, and that priors of current PFNs are unlikely to be separable, so that we should expect $\pi \neq \bar\pi$. 

To better understand this difference, we first characterize how observing the query feature changes the predictive at a fixed context, and then examine what remains as the context grows. At finite context sizes, the two predictives differ as follows (proof in \cref{proof:covariance-shift}).

\begin{corollary}
\label{cor:covariance-shift}
Let $\prior_n = \prior(\cdot \mid \dn)$, with $d\prior_n(p) \propto \prod_{i=1}^n p_X(\bm{x}_i)\,\task{y_i\mid\bm{x}_i} d\prior(p)$, be the posterior given the context, and let $m(\bm{x} \mid \dn) = \E_{\prior_n}[p_X(\bm{x})]$ be the feature predictive. Then
\begin{equation}
    \predgap := \ppd{y \mid \bm{x}, \dn} - \ppdbar{y \mid \bm{x}, \dn}
    = \frac{\operatorname{Cov}_{\prior_n}\bigl(\task{y \mid \bm{x}},\, p_X(\bm{x})\bigr)}{m(\bm{x} \mid \dn)} ,
    \label{eq:covariance-shift}
\end{equation}
where $\predgap$ denotes the predictive gap between the two predictives.
\end{corollary}

The \emph{predictive gap} can be interpreted as a correction term of the context-updated predictive to account for the posterior covariance between the label density at $\bm{x}$ and the conditional feature likelihood $p_X(\bm{x})$ of the query, which it otherwise ignores. Observing the query feature reweights the posterior toward tasks under which that feature value is likely. The query-updated predictive $\pi$ therefore assigns more probability to a label than $\bar\pi$ exactly when these upweighted tasks also predict that label at the query more strongly. Under a separable prior, the covariance is always zero.

The size of $\predgapnoarg$ is not monotone in $n$ or $m(\bm{x} \mid \dn)$ in general. A small feature predictive with a small covariance gives a small correction, and a diffuse posterior need not induce a nonzero covariance. It is large where tasks with comparable posterior weight disagree both in their feature likelihood and in their prediction at $\bm{x}$.

\begin{figure}
    \centering
    \includegraphics[width=0.75\linewidth]{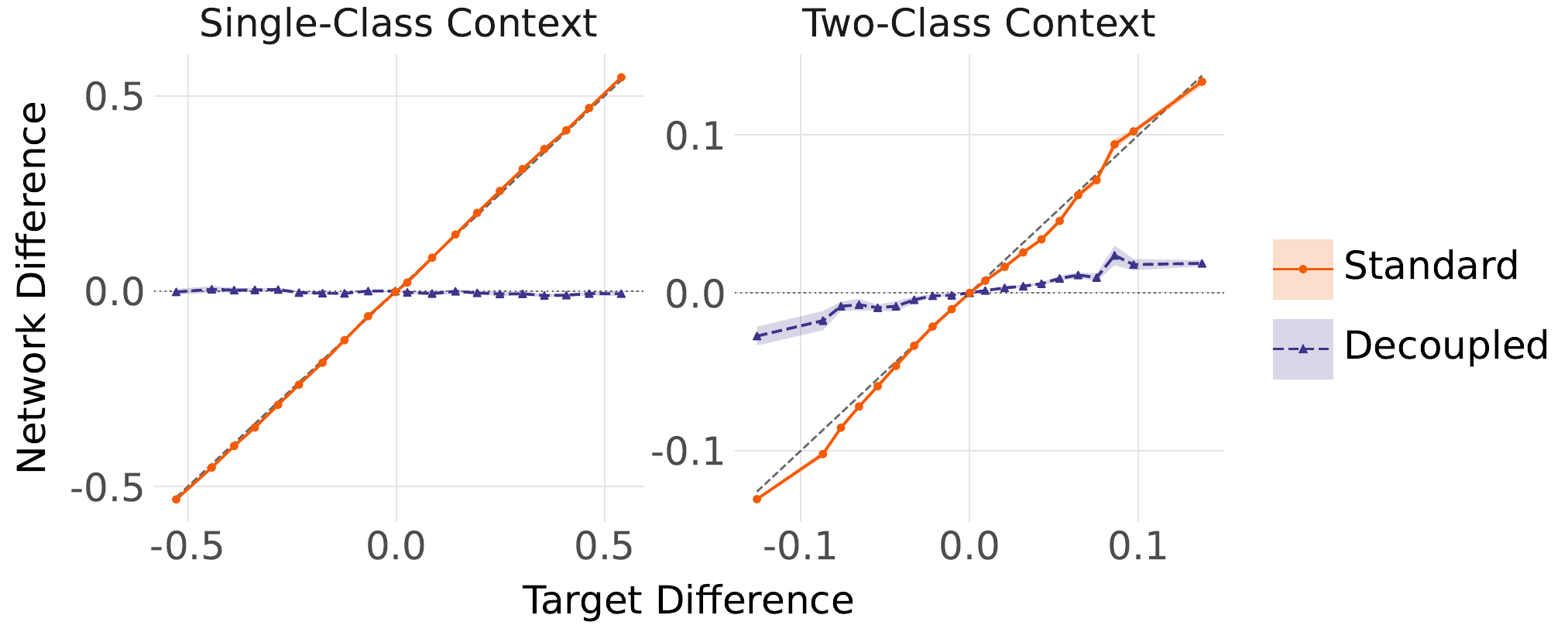}
    \caption{Recovery of the query-updated predictive on the non-separable ordered-cluster prior with context length $n=2$. Standard pretraining follows the identity line, indicating recovery of $\pi$, whereas the difference for the decoupled PFN remains near zero, indicating recovery of $\bar{\pi}$. In single-class contexts, both labeled rows come from the same cluster, so the query is evidence about the unobserved center and the difference between $\pi$ and $\bar\pi$ is larger (note the different axis ranges).}
    \label{fig:approximation_target_exp}
\end{figure}

\Cref{fig:approximation_target_exp} tests which predictive standard PFN pretraining recovers. We pretrain TabICL, denoted by $\pfn{\cdot}$, on a non-separable ordered-cluster prior, whose finite task space makes $\pi$ and $\bar\pi$ exactly computable (\cref{app:exp_ordered_cluster_prior}). Under standard pretraining, the Bayes-optimal target is $\pi$. In contrast, under a decoupled setting, each query label is drawn from a fresh task $p'\sim\prior(\cdot\mid\dn)$, so the target is $\bar\pi$. At $n=2$, the deviation of the learned query-updated predictive from the context-updated predictive, $\pfn{1\mid\bm{x},\dn}-\ppdbar{1\mid\bm{x},\dn}$, follows the exact difference $\predgapnoarg(1\mid\bm{x},\dn)$ along the identity line. The decoupled PFN's deviation remains near zero. Each PFN thus approximates its own Bayes-optimal target. 

\subsection{Consistency}

The covariance in (\ref{eq:covariance-shift}) describes the correction at a given context. Under posterior concentration, this correction vanishes for typical queries as the number of context rows grows. The following result makes this precise. TV and $H$ denote the total variation and Hellinger distances (proof in \cref{proof:nonseparable-consistency}).

\begin{proposition}
\label{prop:nonseparable-consistency}
Let each task $p$ have joint density $p(y, \bm{x})$, and let $\Dn$ consist of iid rows from the true joint density $p_0$. Write $\prior_n = \prior(\cdot \mid \Dn)$; all almost-sure statements refer to the iid context sequence generated by $p_0$. Suppose there is a joint density $p^\ast = p^\ast_{Y|X}p^\ast_X$ such that, for every $\epsilon > 0$,
\begin{align*}
    \prior_n\bigl(\{p : H(p, p^\ast) > \epsilon\}\bigr)
    &\to 0
    \qquad \text{almost surely}.
\end{align*}
Set $\delta_n = \int \mathrm{TV}(p, p^\ast)\, d\prior_n(p)$. Then $\delta_n \to 0$ almost surely, and
\begin{align*}
    \int \mathrm{TV}\bigl(\breve\pi(\cdot \mid \bm{x}, \Dn), p^\ast_{Y|X}(\cdot \mid \bm{x})\bigr)
        p^\ast_X(\bm{x})\, d\bm{x}
    &\le 2\delta_n\, , \quad \breve\pi\in\{\pi,\bar\pi\}.
\end{align*}
\end{proposition}

By the triangle inequality, the integrated total variation between $\pi$ and $\bar\pi$ is at most $4\delta_n$. Thus, the query correction in \cref{eq:covariance-shift} vanishes for features drawn from $p^\ast_X$. If the true feature law satisfies $p_{0,X} \ll p^\ast_X$, this also gives convergence of all three integrated distances for queries drawn from $p_{0,X}$.
When $p^\ast$ is the unique KL-optimal joint density, both predictives approach its conditional distribution. The query correction therefore vanishes for typical queries as the context grows. 

\subsection{Misspecification}

Although the gap between $\pi$ and $\bar \pi$ vanishes with the sample size, non-separability of the prior can still affect the common limit. To see why, let $p_0$ denote the true joint density. The joint KL divergence decomposes as
\begin{align*}
     \KL(p_0\,\Vert\,p) &= \KL(p_{0,X}\,\Vert\,p_X)+\KL^{Y\mid X}(p_0\,\Vert\,p),
\end{align*}
where the expected conditional KL divergence,
\begin{align*}
        \KL^{Y\mid X}(p_0\,\Vert\,p)
    &:= \E_{p_{0,X}}\bigl[\KL\bigl(p_0(\cdot\mid\bm{X})\,\Vert\,p(\cdot\mid\bm{X})\bigr)\bigr],
\end{align*}
measures loss in predictive performance relative to the true  mechanism. When the posterior concentrates at a joint KL minimizer $p^\ast$, this limiting law minimizes the sum of the two terms. In particular, it can favor a mechanism that predicts labels less accurately if that mechanism better explains the features. This tradeoff disappears under separability, but can otherwise lead to bad models.

\begin{proposition}
\label{prop:misspecified-limit}
Let $p_0$ be the true joint density, and let $p^\ast$ minimize $\KL(p_0\,\Vert\,p)$ over the support of $\prior$, with $\KL(p_0\,\Vert\,p^\ast)<\infty$.
\begin{enumerate}[label=(\roman*),leftmargin=2em]
    \item If $\prior$ is separable, then $\KL^{Y\mid X}(p_0\,\Vert\,p^\ast)=\inf_{p\in\operatorname{supp}\prior}\KL^{Y\mid X}(p_0\,\Vert\,p)$.
    \item If $\Pi$ is non-separable, the equality in~(i) need not hold, and the gap can be arbitrarily large.
\end{enumerate}
\end{proposition}

The proof can be found in \cref{proof:misspecified-limit}.
Part~(i) shows that a separable prior attains the best conditional label fit available in its support. Thus, the prior's feature marginal $\Pi_X$ does not affect how well the posterior fits the labels. Under non-separable priors, this is no longer guaranteed. Part~(ii) shows that, in this case, a miscalibrated $\Pi_X$ can make the posterior's label predictions arbitrarily bad. We refer to this as the \emph{optimality gap}.

This tradeoff connects the common limit to the distinction between generative and discriminative estimation under misspecification~\citep{liang2008asymptotic}: the former optimizes joint fit, whereas the latter targets conditional fit. It is also related to the bias changes through which additional unlabeled data can degrade generative classifiers~\citep{cozman2003semi}, although the proposition concerns the shared limit with a single query row. If $p_0$ is a supported joint density, then $\KL^{Y\mid X}(p_0\,\Vert\,p^\ast)=0$. Thus, more context can make $\pi$ and $\bar\pi$ agree without yielding the best label predictions available under the prior: under misspecification, both can favor better feature fit at the expense of label accuracy.

\subsection{What does this tell us?}

PFN pretraining targets the query-updated predictive $\pi$ (\cref{prop:nonseparable-minimizer}), which coincides with the context-updated predictive $\bar\pi$ assumed in the literature when the feature likelihood cancels from the posterior over the labeling mechanism, as it does under separability (\cref{def:separable}). However, even simple SCM priors do not have this property (\cref{prop:scm-nonseparable}) and neither does a mixture of (potentially separable) priors if the regimes do not share their prior over feature distributions or over labeling mechanisms (\cref{prop:mixture-nonseparability}). Practically, at finite context sizes, the query can therefore update beliefs about the labeling mechanism, producing a \emph{predictive gap} between what the PFN community reports to target and the query-updated predictive shown to be the actual target  (\cref{prop:nonseparable-minimizer,eq:covariance-shift}). Under posterior concentration, this correction vanishes, and both predictives approach the same conditional distribution (\cref{prop:nonseparable-consistency}). Under misspecification, however, restrictions in the prior's support can make that common limit sacrifice label fit for feature fit (\cref{prop:misspecified-limit}), resulting in an \emph{optimality gap}. Both gaps can be seen as a \emph{separability gap}, causing (potentially arbitrarily large) differences between the predictives under non-separable priors. 

For PFN design, the prior's feature generator therefore affects the predictions themselves. A PFN pretrained on an SCM prior infers the labeling mechanism partly from the query's feature values, and it can use additional unlabeled rows the same way (as demonstrated in the following \cref{sec:semitabicl}). Whether this improves or degrades predictive performance on real data depends on whether the prior couples feature and label laws as the data do.

\section{Consequences of non-separable priors} 

We have shown that under a prior with coupled feature and label laws, the query can update the posterior over the labeling mechanism. This has many important consequences.

\begin{figure}
    \centering
    \includegraphics[width=\linewidth]{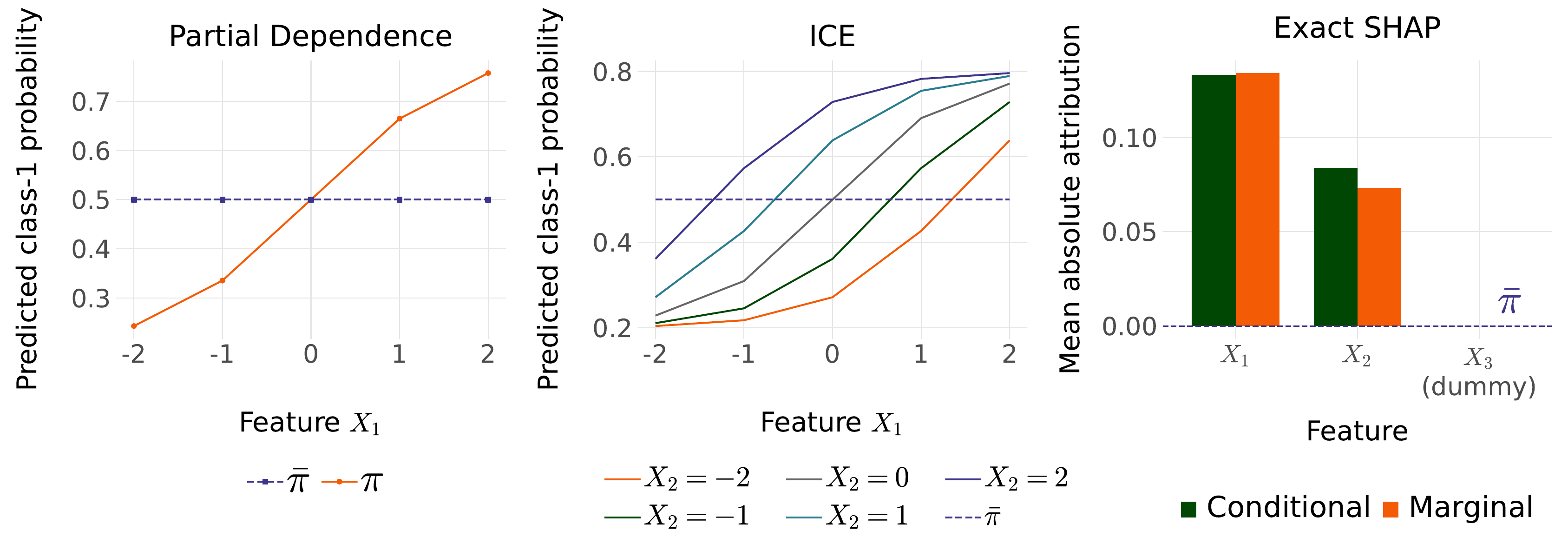}
 \caption{Feature effects under a prior in which every task has a constant class-1 probability $c_p$ that is associated with its feature distribution. There are two tasks, a low-$c_p$ and a high-$c_p$ one. $X_1$ and $X_2$ are discrete features shifted toward negative values in the former and toward positive values in the latter, so they indicate the task, while $X_3$ is independent noise with the same distribution in both tasks. Left: PD of $X_1$. Middle: ICE curves of $X_1$ for fixed $X_2$. Right: mean absolute exact SHAP attributions of $\pi$, conditional and marginal. Since no labeling mechanism depends on $\bm{x}$, $\bar\pi$ is constant and all its attributions are zero (dashed). $\pi$ assigns importance to $X_1$ and $X_2$ because they identify the task, and none to the dummy feature $X_3$.}
    \label{fig:interpretability}
\end{figure}

\paragraph{Interpretability methods}
Feature effect methods for model interpretability are also applicable to PFNs, where batched inference can bring down the cost of repeated model evaluations \citep{rundel2024interpretable}. Many interpretability methods, such as partial dependence (PD) and individual conditional expectation (ICE) plots vary selected features while averaging over, or fixing, the remaining coordinates~\citep{goldstein2015peeking}.  For a PFN approximating $\pi$, varying $\bm{x}$ changes both $p(y\mid\bm{x})$ and the posterior $\prior(\cdot\mid\bm{x},\dn)$ over tasks (\cref{eq:covariance-shift}). The resulting curve therefore combines variation within labeling mechanisms with evidence about which mechanism generated the query. For example, if every task has a constant class probability $c_p := p(1\mid\bm{x})$, but $c_p$ is associated with the task's feature distribution, then $\pi(1\mid\bm{x},\dn)$ can vary with $\bm{x}$ even though $\bar\pi(1\mid\bm{x},\dn)$ is constant, as illustrated in \Cref{fig:interpretability}. This effect can arise within the data support, in addition to the familiar extrapolation problem of PD plots~\citep{apley2020visualizing}. The marginalization convention of SHapley Additive exPlanations \citep[SHAP;][]{lundberg2017unified} does not remove it either. Conditional SHAP averages missing coordinates under their distribution given the retained coordinates, whereas marginal SHAP averages under their unconditional distribution~\citep{molnar2025}. Both still evaluate the same query-updated predictor at each completed input. Thus, these methods can faithfully explain the PFN's predictions while attributing importance to features that identify the task rather than change its conditional label distribution.

\paragraph{PFN for Bayesian optimization}

PFNs are increasingly used as surrogate models in Bayesian optimization \citep[BO;][]{muller2023pfns4bo,yu2026gitbo,rogers2026zeroshot}. BO needs the predictive distribution of the response $y$ at a candidate $\bm{x}$ chosen from previous evaluations $\dn$: $\pi_{\mathrm{BO}}(y\mid\bm{x},\dn)$. If each input is selected using the observed history and randomization independent of $p$, and an evaluation at $\bm{x}$ follows $p(y\mid\bm{x})$, this predictive is
\begin{align*}
    \pi_{\mathrm{BO}}(y\mid\bm{x},\dn)
    = \int p(y\mid\bm{x})\,d\prior_{\mathrm{BO}}(p\mid\dn), \quad
    \text{where} \;\; d\prior_{\mathrm{BO}}(p\mid\dn)
    \propto \prod_{i=1}^n p(y_i\mid\bm{x}_i)\,d\prior(p).
\end{align*}
Under separability, the feature likelihood $p_X(\bm{x})\prod_{i=1}^n p_X(\bm{x}_i)$ in the PFN posterior depends only on $p_X$ (see \cref{prop:nonseparable-minimizer} in \cref{app:proofs}). It cancels when we marginalize over the feature law, so $\pi=\pi_{\mathrm{BO}}$. Therefore, the PFN provides exactly the predictive needed for BO.

Under non-separability, the feature likelihood can also update the labeling mechanism. Standard PFN pretraining thus treats both the previous inputs and the candidate as task-generated evidence, even though the BO algorithm chose them. As a result, its target $\pi$ is usually different from $\pi_{\mathrm{BO}}$. Replacing $\pi$ by $\bar\pi$ removes the influence of the new candidate, but not the influence of the history. In either case, a PFN trained on a non-separable prior does not provide the predictive needed for BO.

\paragraph{Amortization gap}\label{consequence:amortizationGap}
Following \citet{hulskamp2026on}, the amortization gap of a PFN is its excess over the Bayes-optimal predictor: $\gap(\bm{x},\dn)=\KL\bigl(\ppd{\cdot\mid\bm{x},\dn}\,\Vert\,\pfn{\cdot\mid\bm{x},\dn}\bigr)$. Since $\pi$ minimizes the pretraining loss (\cref{prop:nonseparable-minimizer}), the expected gap equals the excess pretraining loss of the PFN (\cref{rem:amortization}). When analyzing how closely a given PFN approximates its target posterior predictive distribution, as, e.g., in \citet{muller_2022_TransformersCana,muller2024bayes} and \citet{hulskamp2026on}, it must be carefully evaluated which predictive from \cref{eq:ppd} is targeted precisely. A reference computed from the context posterior alone, such as an unweighted average over MCMC draws from $\prior(\cdot\mid\dn)$, estimates $\bar\pi$ rather than $\pi$. The discrepancy between this reference and the PFN decomposes into
\begin{align*}
     \pfn{y\mid\bm{x},\dn}-\ppdbar{y\mid\bm{x},\dn}
    &= \bigl(\pfn{y\mid\bm{x},\dn}-\ppd{y\mid\bm{x},\dn}\bigr) + \predgap,
\end{align*}
so any divergence computed against $\bar \pi$ is corrupted by the predictive gap of \cref{cor:covariance-shift}. In the example of \cref{fig:approximation_target_exp}, this overstates the gap of the standard PFN by a factor of about $95$. When $p_X$ can be evaluated, weighting each draw by the query's feature likelihood removes the bias and can provide the correct reference $\pi$. Under a separable prior, no correction is needed since $\pi=\bar\pi$.

\paragraph{Martingale checks}
A martingale check, as employed in the context of PFNs \citep{falck2024martingale,ng2026tabmgp,fong2023martingale}, asks whether generating additional observations $Y_{n+1}\sim\pi(\cdot\mid\bm{x}_{n+1},\dn)$ and updating on them (e.g., by updating the context $D_{n+1}=(\dn,(Y_{n+1},\bm{x}_{n+1}))$) leaves a prediction at a fixed query unchanged on average. Under a non-separable prior, the predictor treats externally generated features as evidence about the labeling mechanism, so even an exact Bayesian predictor can drift; i.e., in general it holds
\begin{equation}
    \E_{Y_{n+1}\sim\pi(\cdot\mid\bm{x}_{n+1},\dn)}
        \bigl[\pi(A\mid\bm{x},D_{n+1})\bigr]
        \ne \pi(A\mid\bm{x},\dn).
        \label{eq:martingale-check}
\end{equation}
The desired equality holds under separability. Without it, the appended feature provides evidence about the task, and the label sampler omits conditioning on the tracked query. Thus, drift can arise without PFN approximation error, although non-separability does not force drift at every query. For the underlying argument, refer to \cref{app:martingale}. Resampling features with a Bayesian bootstrap \citep{rubin1981bayesian,ng2026tabmgp} averages these drifts without guaranteeing that they cancel.

To illustrate this failure, we evaluate the exact predictives $\pi$ and $\bar\pi$ of the non-separable ordered-cluster prior (\cref{app:exp_ordered_cluster_prior}) on $2000$ contexts of size $n=2$, with eight queries each. Future features that are then appended to the context are fixed at the query (i.e., each new context row is a labeled copy of the query location) or at one context feature selected once per context-query pair. Labels are drawn sequentially from the respective predictive where the query stays fixed throughout, and each appended row updates the context through its joint likelihood. We measure cumulative expected drift as the expected class-1 prediction after the appended rows minus its initial value, holding the initial context, query, and repeated future feature fixed. For details, refer to \cref{app:martingale}. The results in \Cref{fig:martingale} show substantial drift in both predictives. This demonstrates that the exact target can drift under the diagnostic's sampling protocol, irrespective of the network's approximation error.

\begin{figure}
    \centering
    \begin{subfigure}[b]{0.48\textwidth}
        \centering
        \includegraphics[width=\textwidth]{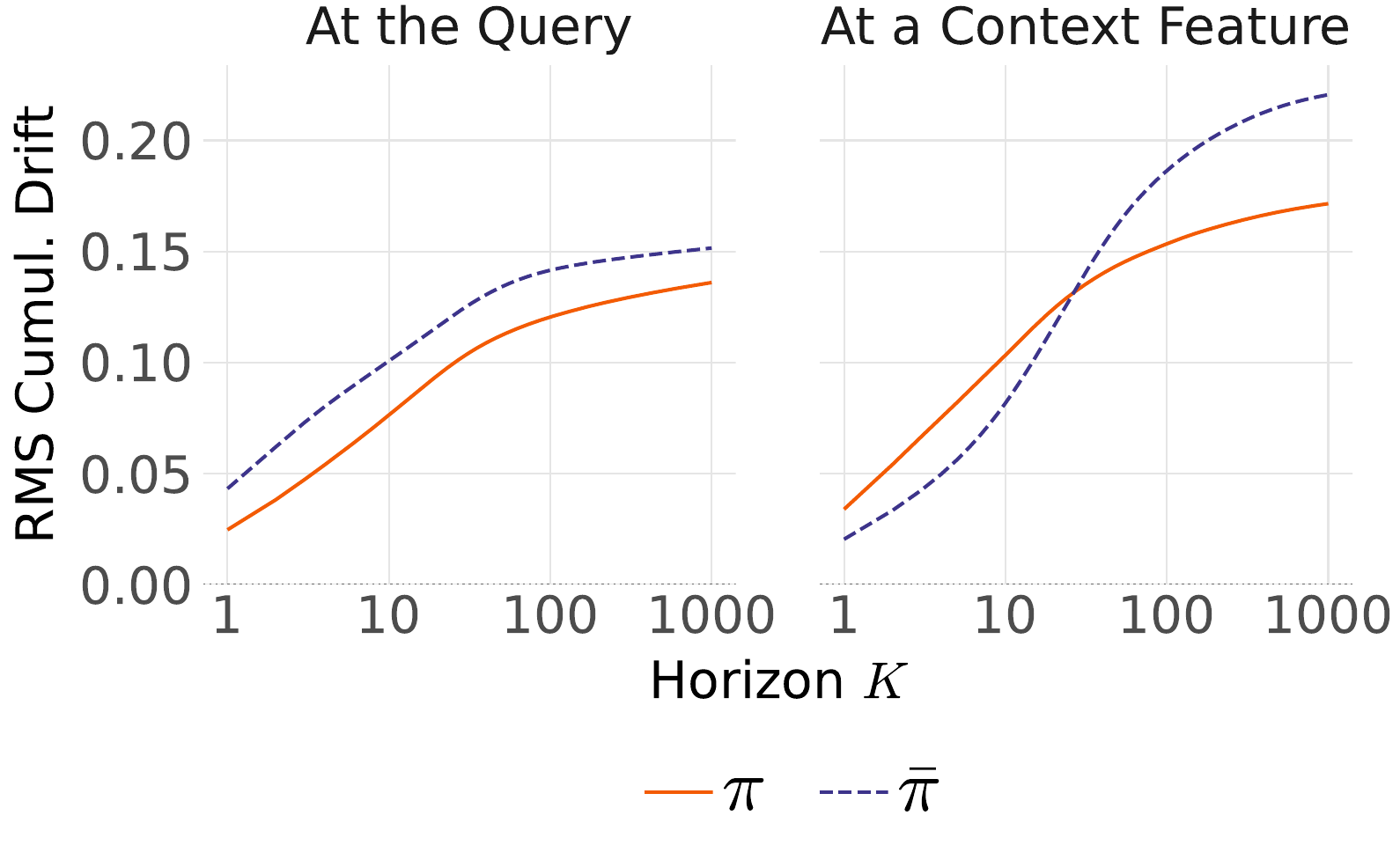}
    \end{subfigure}
    \hfill
    \begin{subfigure}[b]{0.48\textwidth}
        \centering
        \includegraphics[width=\textwidth]{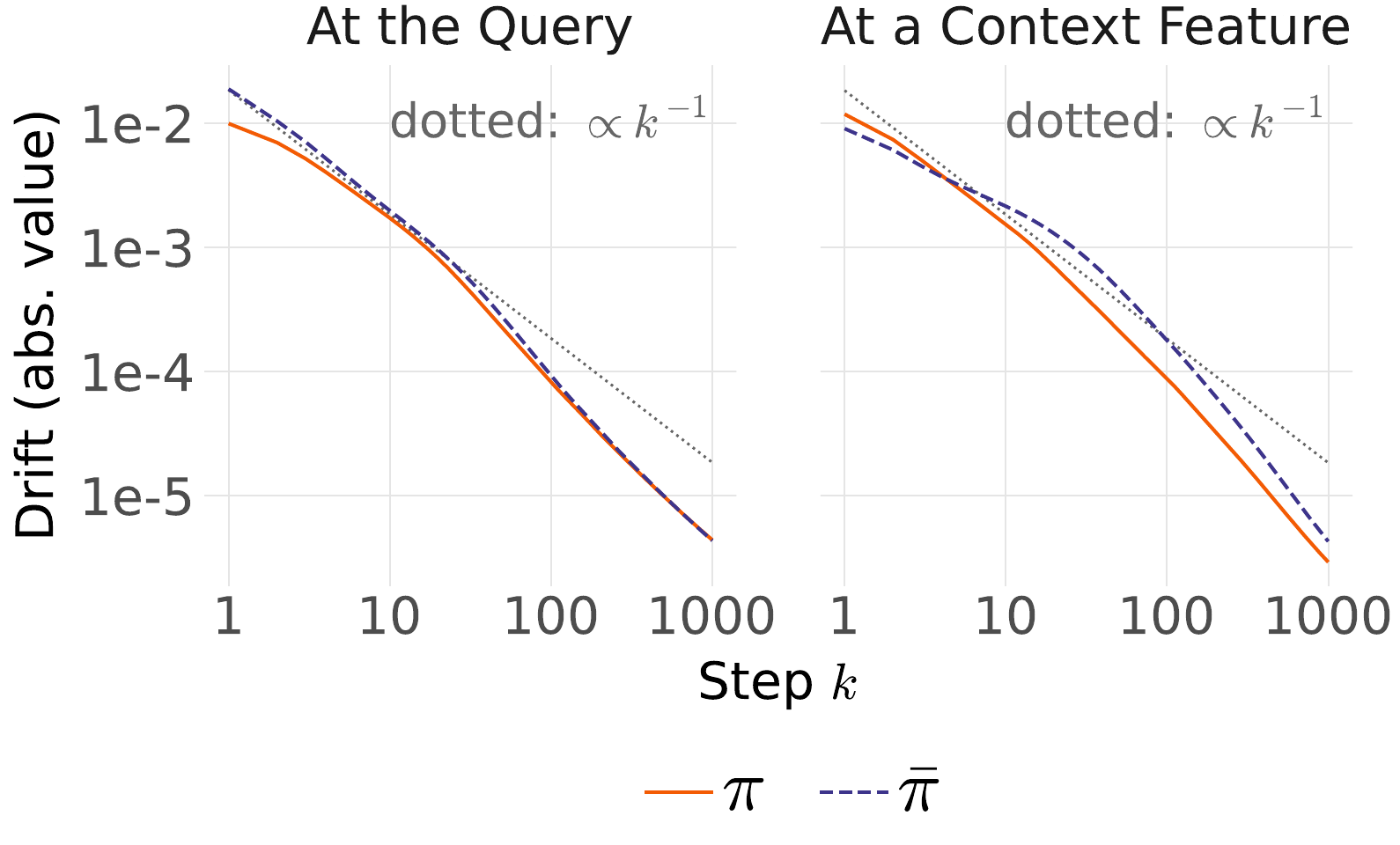}
    \end{subfigure}
    \caption{Non-separable priors can induce drift in exact Bayesian predictives. Curves show the RMS of the cumulative expected drift across $2000$ contexts of size $n=2$, with eight queries each, under the ordered-cluster prior. Every future feature is fixed at the query (left) or at one selected context feature (right), and labels are sampled from the predictive being evaluated with absolute conditional one-step drifts in the right figure. No trained network is involved.}
    \label{fig:martingale}
\end{figure}

\paragraph{PFNs as semi-supervised predictors}
\label{sec:semitabicl}
By \cref{cor:unlabeled-predictive}, PFNs pretrained on non-separable priors are already semi-supervised predictors: $\pi$ uses the unlabeled query $\bm{x}$ as evidence about the labeling mechanism. While they are currently limited to only use the additional evidence in the single query row, there is no structural reason not to integrate further unlabeled rows $\bm{V}$ into the prediction, turning PFNs into full semi-supervised predictors. \citet{schoelkopf2012causal} argue that this addition of unlabeled data should help especially in anti-causal and confounded tasks, which SCM priors already generate. Hence, integrating such unlabeled rows $\bm{V}$ into existing training pipelines and architectures should prove beneficial, which we show for both the ordered-cluster prior (see \cref{app:semi-ordered-cluster}) and the full-sized production prior of TabICLv2.

As such, we pretrained a novel 8.5M-parameter TabICL variant, \textit{SemiTabICL}, to not just use the unlabeled information in the query $\bm{x}$, but instead also explicitly allow for the incorporation of an additional set of unlabeled context rows $\bm{V}$ which are processed exactly as regular context rows, only with an additional learned embedding vector instead of the label embedding (see \cref{app:semitabicl}). Following \citet{guo2010extensive}, we evaluate SemiTabICL in the standard setting of semi-supervised learning with a low number of labeled training points relative to a large number of unlabeled training points. To this end, we reuse the TALENT benchmark \citep{liu2025talent} and convert it into a benchmark for semi-supervised classification, that we call \textit{semi-TALENT}, by splitting the official training fold of the classification datasets into labeled partitions of $n\in\{16,32,64,96\}$ rows and unlabeled partitions across multiple repetitions. We compare SemiTabICL to a baseline TabICL model pretrained without additional unlabeled context rows, as well as a self-training PFN variant where unlabeled rows predicted with high confidence are iteratively appended with their predicted labels before the prediction on the test set is made \citep{yarowsky1995unsupervised}.  

\Cref{fig:semiPFN-dose-response} shows the differences in accuracy for these two settings across a range of labeled training rows $n$ and the multiples thereof added as additional unlabeled rows $k$. SemiTabICL consistently improves upon the baseline by 0.4--1.1 percentage points, while the self-training hurts performance in almost every setting, increasingly so with more labeled rows. \Cref{app:semitabicl} contains further details on the evaluation as well as another comparison on the UCI datasets used in \citet{schoelkopf2012causal,guo2010extensive}. Importantly, when incorporating the capability to use additional unlabeled rows, it must be carefully specified whether these may coincide with the test rows, a point we will cover in the following paragraph.

\begin{figure}
    \centering
    \includegraphics[width=0.65\linewidth]{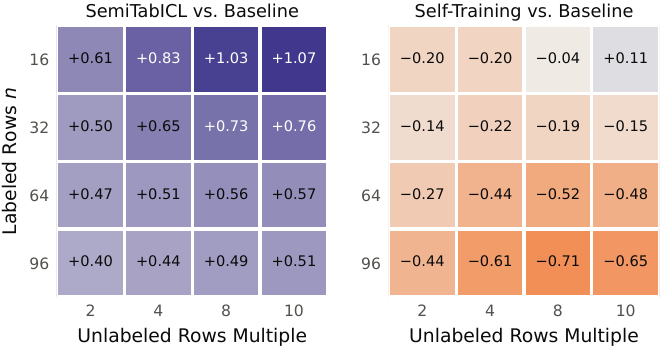}
    \caption{Accuracy gains (percentage points) for semi-supervised classification on the semi-TALENT benchmark for a range of labeled context dataset sizes ($n$) and multiples thereof of unlabeled context rows.  Left: SemiTabICL vs.\ identical baseline trained and evaluated without unlabeled context rows. Right: Semi-supervised extension of the baseline model via self-training vs.\ no additional self-training}
    \label{fig:semiPFN-dose-response}
\end{figure}

\paragraph{Evaluation protocols}
Once PFNs condition on unlabeled rows, the distinction between inductive and transductive inference matters for evaluation \citep{schoelkopf2012causal}. An inductive model predicts each test row on its own, whereas a transductive one predicts the test rows jointly, so the prediction for one row may depend on the others. Semi-supervised predictions induced by training on non-separable priors or explicitly for SemiTabICL stay inductive as long as unlabeled and other test rows are disjoint. Once they overlap, a method gains information that an inductive competitor does not have, which can lower the Bayes-optimal loss. Current benchmarks serve the test data in different ways, but none of them ensures that the prediction for one test row cannot use the others. No current tabular foundation model is designed to exploit this. Still, several pipelines already let test rows affect each other's predictions through shared preprocessing or compute optimizations (\cref{rem:transductive}). Benchmarks should therefore state whether transductive inference is permitted and, if it is not, enforce invariance to the other test rows throughout preprocessing and prediction.

\section{Discussion}
\label{sec:conclusion}

This article has highlighted a subtle but important discrepancy between the typical formulation of the posterior predictive distribution in the PFN literature and the actual target of PFN pretraining under non-separable priors. We further showed that this discrepancy has several practical implications for interpretability, Bayesian optimization, martingale checks, and semi-supervised learning.
Depending on the application, these effects can be unwanted or beneficial. We argue that the distinction between the query-updated predictive $\pi$ and the context-updated predictive $\bar\pi$ is important when designing PFN pretraining and evaluation protocols, especially in the presence of non-separable priors. Although this difference vanishes under the stated concentration conditions, coupling feature and label distributions can still sacrifice label fit for feature fit under misspecification. Our experiments confirm recovery of the query-updated target and show that explicitly incorporating additional unlabeled rows yields predictive gains with SemiTabICL on real tabular data.

Separable priors avoid the issues discussed in this work, but also forgo potential gains from using unlabeled features. Our analysis already points to useful design principles for such priors. SCM priors can remain the basis for generative modeling of tabular data, but two independent SCM mechanisms should be drawn, one for generating the features and one for generating the label. One could also pretrain on a range of coupling strengths between the mechanisms.
These findings make the dependence between feature and label laws a concrete design choice for PFNs and a promising direction for future work.

\paragraph{Limitations} The main limitation of this work is that we mainly establish that non-separability \emph{can} have the issues discussed in this work. We do not yet know whether production-grade PFNs benefit or suffer from it in practice, and to what extent. Studying this is challenging, since available PFNs use diverse pretraining priors and many are not open-source. Further, we can only compare to meaningful comparator models after full-scale retraining of the model under a separable sibling of its prior. We hope that our work motivates future research in this direction.

\subsection*{AI use statement}
In this work, we used generative AI tools to support the implementation of the presented methods in code and to draft and review mathematical proofs.
We have not used generative AI tools to formulate research hypotheses and develop the theoretical and conceptual framework. Furthermore, other required disclosures such as data cleaning or data analysis are not applicable to this work.
Additionally, we used generative AI tools to assist in the writing and presentation of this manuscript by improving readability and the presentation of the included figures. Furthermore, generative AI tools were used to explore and identify relevant literature in addition to our initial manual survey. We have reviewed all AI-assisted work. AI-generated code was reviewed manually as well as tested in detail. All modifications of the manuscript suggested by AI tools were carefully inspected. We take responsibility for the final content of this work, including text, claims or artifacts produced with the aid of generative AI.

\subsubsection*{Acknowledgments}
This work was supported by the Bavarian AI Foundation Model Initiative (AI-Bay), funded by the Bavarian State Ministry of Science and the Arts. The authors gratefully acknowledge the computational resources and support provided by the Erlangen National High Performance Computing Center (NHR@FAU) at Friedrich-Alexander-Universität Erlangen-Nürnberg (FAU).

\bibliography{references}

\begin{thebibliography}{47}
\providecommand{\natexlab}[1]{#1}
\providecommand{\url}[1]{\texttt{#1}}
\expandafter\ifx\csname urlstyle\endcsname\relax
  \providecommand{\doi}[1]{doi: #1}\else
  \providecommand{\doi}{doi: \begingroup \urlstyle{rm}\Url}\fi

\bibitem[Apley \& Zhu(2020)Apley and Zhu]{apley2020visualizing}
Daniel~W Apley and Jingyu Zhu.
\newblock Visualizing the effects of predictor variables in black box supervised learning models.
\newblock \emph{Journal of the Royal Statistical Society Series B: Statistical Methodology}, 82\penalty0 (4), 2020.

\bibitem[Cozman et~al.(2003)Cozman, Cohen, and Cirelo]{cozman2003semi}
Fabio~G Cozman, Ira Cohen, and Marcelo~C Cirelo.
\newblock {Semi-supervised learning of mixture models}.
\newblock In \emph{Proceedings of the 20th international conference on machine learning}, 2003.

\bibitem[De~Blasi \& Walker(2013)De~Blasi and Walker]{deblasiwalker2013}
Pierpaolo De~Blasi and Stephen~G. Walker.
\newblock Bayesian asymptotics with misspecified models.
\newblock \emph{Statistica Sinica}, 23\penalty0 (1):\penalty0 169--187, 2013.
\newblock \doi{10.5705/ss.2010.239}.

\bibitem[Erickson et~al.(2026)Erickson, Purucker, Tschalzev, Holzm{\"u}ller, Desai, Salinas, and Hutter]{erickson2026tabarena}
Nick Erickson, Lennart Purucker, Andrej Tschalzev, David Holzm{\"u}ller, Prateek Desai, David Salinas, and Frank Hutter.
\newblock {TabArena: A living benchmark for machine learning on tabular data}.
\newblock \emph{Advances in Neural Information Processing Systems}, 2026.

\bibitem[Falck et~al.(2024)Falck, Wang, and Holmes]{falck2024martingale}
Fabian Falck, Ziyu Wang, and Christopher~C. Holmes.
\newblock Is in-context learning in large language models bayesian? {A} martingale perspective.
\newblock In \emph{Proceedings of the 41st International Conference on Machine Learning}, Proceedings of Machine Learning Research. PMLR, 2024.

\bibitem[Fong et~al.(2023)Fong, Holmes, and Walker]{fong2023martingale}
Edwin Fong, Chris Holmes, and Stephen~G Walker.
\newblock Martingale posterior distributions.
\newblock \emph{Journal of the Royal Statistical Society Series B: Statistical Methodology}, 85\penalty0 (5):\penalty0 1357--1391, 2023.

\bibitem[Foong et~al.(2020)Foong, Bruinsma, Gordon, Dubois, Requeima, and Turner]{foong2020meta}
Andrew Foong, Wessel Bruinsma, Jonathan Gordon, Yann Dubois, James Requeima, and Richard Turner.
\newblock {Meta-learning stationary stochastic process prediction with convolutional neural processes}.
\newblock \emph{Advances in Neural Information Processing Systems}, 2020.

\bibitem[Garnelo et~al.(2018{\natexlab{a}})Garnelo, Rosenbaum, Maddison, Ramalho, Saxton, Shanahan, Teh, Rezende, and Eslami]{garnelo18a}
Marta Garnelo, Dan Rosenbaum, Christopher Maddison, Tiago Ramalho, David Saxton, Murray Shanahan, Yee~Whye Teh, Danilo Rezende, and S.~M.~Ali Eslami.
\newblock Conditional neural processes.
\newblock In \emph{Proceedings of the 35th International Conference on Machine Learning}. PMLR, 2018{\natexlab{a}}.

\bibitem[Garnelo et~al.(2018{\natexlab{b}})Garnelo, Schwarz, Rosenbaum, Viola, Rezende, Eslami, and Teh]{garnelo2018neural}
Marta Garnelo, Jonathan Schwarz, Dan Rosenbaum, Fabio Viola, Danilo~J Rezende, SM~Eslami, and Yee~Whye Teh.
\newblock {Neural processes}.
\newblock \emph{arXiv preprint arXiv:1807.01622}, 2018{\natexlab{b}}.

\bibitem[Goldstein et~al.(2015)Goldstein, Kapelner, Bleich, and Pitkin]{goldstein2015peeking}
Alex Goldstein, Adam Kapelner, Justin Bleich, and Emil Pitkin.
\newblock {Peeking inside the black box: Visualizing statistical learning with plots of individual conditional expectation}.
\newblock \emph{Journal of Computational and Graphical Statistics}, 24\penalty0 (1), 2015.

\bibitem[Grinsztajn et~al.(2025)Grinsztajn, Flöge, Key, Birkel, Jund, Roof, Jäger, Safaric, Alessi, Hayler, Manium, Yu, Jablonski, Hoo, Garg, Robertson, Bühler, Moroshan, Purucker, Cornu, Wehrhahn, Bonetto, schoelkopf, Gambhir, Hollmann, and Hutter]{grinsztajn2025tabpfnv25}
Léo Grinsztajn, Klemens Flöge, Oscar Key, Felix Birkel, Philipp Jund, Brendan Roof, Benjamin Jäger, Dominik Safaric, Simone Alessi, Adrian Hayler, Mihir Manium, Rosen Yu, Felix Jablonski, Shi~Bin Hoo, Anurag Garg, Jake Robertson, Magnus Bühler, Vladyslav Moroshan, Lennart Purucker, Clara Cornu, Lilly~Charlotte Wehrhahn, Alessandro Bonetto, Bernhard schoelkopf, Sauraj Gambhir, Noah Hollmann, and Frank Hutter.
\newblock {TabPFN-2.5: Advancing the state of the art in tabular foundation models}, 2025.

\bibitem[Grinsztajn et~al.(2026)Grinsztajn, Flöge, Key, Birkel, Jund, Roof, Manium, Hoo, Bühler, Garg, Safaric, Robertson, Jäger, Alessi, Hayler, Moroshan, Purucker, Singer, Arazi, Siems, Metzen, Grab, Erickson, Guo, Kalfon, Bing, Salinas, Cornu, Wehrhahn, Kriuchkova, Kaya, Sidhoum, Salmon, Chen, Hulsebos, LeCun, Müller, schoelkopf, Gambhir, Hollmann, and Hutter]{grinsztajn2026tabpfnv3}
Léo Grinsztajn, Klemens Flöge, Oscar Key, Felix Birkel, Philipp Jund, Brendan Roof, Mihir Manium, Shi~Bin Hoo, Magnus Bühler, Anurag Garg, Dominik Safaric, Jake Robertson, Benjamin Jäger, Simone Alessi, Adrian Hayler, Vladyslav Moroshan, Lennart Purucker, Philipp Singer, Alan Arazi, Julien Siems, Jan~Hendrik Metzen, Georg Grab, Nick Erickson, Siyuan Guo, Eliott Kalfon, Simon Bing, David Salinas, Clara Cornu, Lilly~Charlotte Wehrhahn, Diana Kriuchkova, Kursat Kaya, Lydia Sidhoum, Marie Salmon, Jerry Chen, Madelon Hulsebos, Yann LeCun, Samuel Müller, Bernhard schoelkopf, Sauraj Gambhir, Noah Hollmann, and Frank Hutter.
\newblock {TabPFN-3: Technical report}, 2026.

\bibitem[Guo et~al.(2010)Guo, Niu, and Zhang]{guo2010extensive}
Yuanyuan Guo, Xiaoda Niu, and Harry Zhang.
\newblock {An extensive empirical study on semi-supervised learning}.
\newblock In \emph{2010 IEEE International Conference on Data Mining}. IEEE, 2010.

\bibitem[Hollmann et~al.(2023)Hollmann, M{\"u}ller, Eggensperger, and Hutter]{hollmanntabpfn}
Noah Hollmann, Samuel M{\"u}ller, Katharina Eggensperger, and Frank Hutter.
\newblock {TabPFN: A transformer that solves small tabular classification problems in a second}.
\newblock In \emph{The Eleventh International Conference on Learning Representations}, 2023.

\bibitem[Hollmann et~al.(2025)Hollmann, M{\"u}ller, Purucker, Krishnakumar, K{\"o}rfer, Hoo, Schirrmeister, and Hutter]{hollmann2025tabpfnv2}
Noah Hollmann, Samuel M{\"u}ller, Lennart Purucker, Arjun Krishnakumar, Max K{\"o}rfer, Shi~Bin Hoo, Robin~Tibor Schirrmeister, and Frank Hutter.
\newblock Accurate predictions on small data with a tabular foundation model.
\newblock \emph{Nature}, 01 2025.
\newblock \doi{10.1038/s41586-024-08328-6}.

\bibitem[H{\"u}lskamp et~al.(2026)H{\"u}lskamp, Kobialka, Sommer, and R{\"u}gamer]{hulskamp2026on}
Manuel H{\"u}lskamp, Julius Kobialka, Emanuel Sommer, and David R{\"u}gamer.
\newblock On the uncertainty in prior-data fitted network pretraining.
\newblock In \emph{2nd ICML Workshop on Foundation Models for Structured Data}, 2026.

\bibitem[Janzing et~al.(2020)Janzing, Minorics, and Bl{\"o}baum]{janzing2020feature}
Dominik Janzing, Lenon Minorics, and Patrick Bl{\"o}baum.
\newblock {Feature relevance quantification in explainable AI: A causal problem}.
\newblock In \emph{International Conference on artificial intelligence and statistics}. PMLR, 2020.

\bibitem[Jordan et~al.(2024)Jordan, Jin, Boza, You, Cesista, Newhouse, and Bernstein]{jordan2024muon}
Keller Jordan, Yuchen Jin, Vlado Boza, Jiacheng You, Franz Cesista, Laker Newhouse, and Jeremy Bernstein.
\newblock {Muon: An optimizer for hidden layers in neural networks}, 2024.
\newblock URL \url{https://kellerjordan.github.io/posts/muon/}.

\bibitem[Kim et~al.(2019)Kim, Mnih, Schwarz, Garnelo, Eslami, Rosenbaum, Vinyals, and Teh]{kim2018attentive}
Hyunjik Kim, Andriy Mnih, Jonathan Schwarz, Marta Garnelo, Ali Eslami, Dan Rosenbaum, Oriol Vinyals, and Yee~Whye Teh.
\newblock Attentive neural processes.
\newblock In \emph{International Conference on Learning Representations}, 2019.

\bibitem[Landsgesell et~al.(2026)Landsgesell, Knoll, and Wenzel]{landsgesell2026scoringbench}
Jonas Landsgesell, Pascal Knoll, and Tizian Wenzel.
\newblock {ScoringBench: A benchmark for evaluating tabular foundation models with proper scoring rules}.
\newblock \emph{arXiv preprint arXiv:2603.29928}, 2026.

\bibitem[Li et~al.(2026)Li, Chang, Kara, Liu, Roy-Chowdhury, and Oymak]{li2026and}
Yingcong Li, Xiangyu Chang, Muti Kara, Xiaofeng Liu, Amit Roy-Chowdhury, and Samet Oymak.
\newblock {When and how unlabeled data provably improve in-context learning}.
\newblock \emph{Advances in Neural Information Processing Systems}, 2026.

\bibitem[Liang \& Jordan(2008)Liang and Jordan]{liang2008asymptotic}
Percy Liang and Michael~I Jordan.
\newblock {An asymptotic analysis of generative, discriminative, and pseudolikelihood estimators}.
\newblock In \emph{Proceedings of the 25th international conference on Machine learning}, 2008.

\bibitem[Liu et~al.(2025)Liu, Cai, Zhou, Yin, Zhou, Jiang, and Ye]{liu2025talent}
Si-Yang Liu, Hao-Run Cai, Qi-Le Zhou, Huai-Hong Yin, Tao Zhou, Jun-Peng Jiang, and Han-Jia Ye.
\newblock {TALENT: A tabular analytics and learning toolbox}.
\newblock \emph{Journal of Machine Learning Research}, 26, 2025.

\bibitem[Lundberg \& Lee(2017)Lundberg and Lee]{lundberg2017unified}
Scott~M Lundberg and Su-In Lee.
\newblock {A unified approach to interpreting model predictions}.
\newblock \emph{Advances in neural information processing systems}, 2017.

\bibitem[Molnar(2025)]{molnar2025}
Christoph Molnar.
\newblock \emph{{Interpretable Machine Learning}}.
\newblock 3 edition, 2025.
\newblock ISBN 978-3-911578-03-5.
\newblock URL \url{https://christophm.github.io/interpretable-ml-book}.

\bibitem[M{\"u}ller et~al.(2022)M{\"u}ller, Hollmann, Arango, Grabocka, and Hutter]{muller_2022_TransformersCana}
Samuel M{\"u}ller, Noah Hollmann, Sebastian~Pineda Arango, Josif Grabocka, and Frank Hutter.
\newblock Transformers {{Can Do Bayesian Inference}}, 2022.

\bibitem[M{\"u}ller et~al.(2023)M{\"u}ller, Feurer, Hollmann, and Hutter]{muller2023pfns4bo}
Samuel M{\"u}ller, Matthias Feurer, Noah Hollmann, and Frank Hutter.
\newblock {Pfns4bo: In-context learning for bayesian optimization}.
\newblock In \emph{International Conference on Machine Learning}. PMLR, 2023.

\bibitem[M{\"u}ller et~al.(2024)M{\"u}ller, Hollmann, and Hutter]{muller2024bayes}
Samuel M{\"u}ller, Noah Hollmann, and Frank Hutter.
\newblock {Bayes' power for explaining in-context learning generalizations}.
\newblock \emph{arXiv preprint arXiv:2410.01565}, 2024.

\bibitem[Nagler(2023)]{nagler23statistical}
Thomas Nagler.
\newblock {Statistical Foundations of Prior-Data Fitted Networks}.
\newblock In \emph{Proceedings of the 40th International Conference on Machine Learning}, 2023.

\bibitem[Ng et~al.(2026)Ng, Fong, Frazier, Knoblauch, and Wei]{ng2026tabmgp}
Kenyon Ng, Edwin Fong, David~Tyler Frazier, Jeremias Knoblauch, and Susan Wei.
\newblock {Tab{MGP}: Martingale Posterior with Tab{PFN}}.
\newblock In \emph{Forty-third International Conference on Machine Learning}, 2026.

\bibitem[Nguyen \& Grover(2022)Nguyen and Grover]{nguyen2022tnp}
Tung Nguyen and Aditya Grover.
\newblock {Transformer Neural Processes: Uncertainty-Aware Meta Learning Via Sequence Modeling}.
\newblock In \emph{Proceedings of the 39th International Conference on Machine Learning}, Proceedings of Machine Learning Research. PMLR, 2022.

\bibitem[Paszke et~al.(2019)Paszke, Gross, Massa, Lerer, Bradbury, Chanan, Killeen, Lin, Gimelshein, Antiga, Desmaison, Kopf, Yang, DeVito, Raison, Tejani, Chilamkurthy, Steiner, Fang, Bai, and Chintala]{paszke_2019_PyTorchImperative}
Adam Paszke, Sam Gross, Francisco Massa, Adam Lerer, James Bradbury, Gregory Chanan, Trevor Killeen, Zeming Lin, Natalia Gimelshein, Luca Antiga, Alban Desmaison, Andreas Kopf, Edward Yang, Zachary DeVito, Martin Raison, Alykhan Tejani, Sasank Chilamkurthy, Benoit Steiner, Lu~Fang, Junjie Bai, and Soumith Chintala.
\newblock {{PyTorch}}: {{An Imperative Style}}, {{High-Performance Deep Learning Library}}.
\newblock In \emph{Proceedings of the 33rd {{Conference}} on {{Neural Information Processing Systems}}}, 2019.

\bibitem[Pedregosa et~al.(2011)Pedregosa, Varoquaux, Gramfort, Michel, Thirion, Grisel, Blondel, Prettenhofer, Weiss, Dubourg, et~al.]{pedregosa2011scikit}
Fabian Pedregosa, Ga{\"e}l Varoquaux, Alexandre Gramfort, Vincent Michel, Bertrand Thirion, Olivier Grisel, Mathieu Blondel, Peter Prettenhofer, Ron Weiss, Vincent Dubourg, et~al.
\newblock Scikit-learn: Machine learning in python.
\newblock \emph{Journal of Machine Learning Research}, 12, 2011.

\bibitem[Purucker et~al.(2026)Purucker, Tschalzev, Erickson, Blayer, Holzmüller, Arazi, Pfefferle, Tajjar, Varoquaux, and Hutter]{purucker2026beyond}
Lennart Purucker, Andrej Tschalzev, Nick Erickson, Gioia Blayer, David Holzmüller, Alan Arazi, Alexander Pfefferle, Mustafa Tajjar, Gaël Varoquaux, and Frank Hutter.
\newblock Beyond {IID}: How general are tabular foundation models, really?
\newblock \emph{arXiv preprint arXiv:2606.30410}, 2026.

\bibitem[Qu et~al.(2025)Qu, Holzm{\"u}ller, Varoquaux, and Le~Morvan]{qu2025tabicl}
Jingang Qu, David Holzm{\"u}ller, Ga{\"e}l Varoquaux, and Marine Le~Morvan.
\newblock Tab{ICL}: {A} tabular foundation model for in-context learning on large data.
\newblock In \emph{International Conference on Machine Learning}, 2025.

\bibitem[Qu et~al.(2026)Qu, Holzm{\"u}ller, Varoquaux, and Le~Morvan]{qu2026tabiclv2}
Jingang Qu, David Holzm{\"u}ller, Ga{\"e}l Varoquaux, and Marine Le~Morvan.
\newblock {TabICLv2}: {A} better, faster, scalable, and open tabular foundation model.
\newblock In \emph{International Conference on Machine Learning}, 2026.

\bibitem[Rogers \& Ponnada(2026)Rogers and Ponnada]{rogers2026zeroshot}
Theodore Rogers and Srividya Ponnada.
\newblock Zero-shot bayesian optimization with {TabPFN}: Competitive with state-of-the-art without per-task training.
\newblock In \emph{AutoML 2026 Methods Track}, 2026.

\bibitem[Rubin(1981)]{rubin1981bayesian}
Donald~B Rubin.
\newblock {The bayesian bootstrap}.
\newblock \emph{The annals of statistics}, pp.\  130--134, 1981.

\bibitem[Rundel et~al.(2024)Rundel, Kobialka, von Crailsheim, Feurer, Nagler, and R{\"u}gamer]{rundel2024interpretable}
David Rundel, Julius Kobialka, Constantin von Crailsheim, Matthias Feurer, Thomas Nagler, and David R{\"u}gamer.
\newblock {Interpretable machine learning for TabPFN}.
\newblock In \emph{World Conference on Explainable Artificial Intelligence}, 2024.

\bibitem[Sch{\"o}lkopf et~al.(2012)Sch{\"o}lkopf, Janzing, Peters, Sgouritsa, Zhang, and Mooij]{schoelkopf2012causal}
Bernhard Sch{\"o}lkopf, Dominik Janzing, Jonas Peters, Eleni Sgouritsa, Kun Zhang, and Joris Mooij.
\newblock {On Causal and Anticausal Learning}.
\newblock In \emph{Proceedings of the 29th International Coference on International Conference on Machine Learning}, 2012.

\bibitem[Somepalli et~al.(2021)Somepalli, Goldblum, Schwarzschild, Bruss, and Goldstein]{somepalli2021saint}
Gowthami Somepalli, Micah Goldblum, Avi Schwarzschild, C~Bayan Bruss, and Tom Goldstein.
\newblock {Saint: Improved neural networks for tabular data via row attention and contrastive pre-training}.
\newblock \emph{arXiv preprint arXiv:2106.01342}, 2021.

\bibitem[Swelam et~al.(2025)Swelam, Purucker, Robertson, Raum, Boedecker, and Hutter]{swelam2025does}
Omar Swelam, Lennart Purucker, Jake Robertson, Hanne Raum, Joschka Boedecker, and Frank Hutter.
\newblock Does {TabPFN} understand causal structures?
\newblock In \emph{EurIPS 2025 Workshop: AI for Tabular Data}, 2025.

\bibitem[Tao et~al.(2026)Tao, Zhang, Liu, Han, Maddix, Fang, Han, Gai, Liu, Bohlke-Schneider, Wang, Friedland, Mah, Lee, and Kong]{tao2026mitrav2}
Yefan Tao, Xiyuan Zhang, Xinyi Liu, Boran Han, Danielle Maddix, Haoyang Fang, Zhen Han, Jiading Gai, Xuanqing Liu, Michael Bohlke-Schneider, Yuyang Wang, Gerald Friedland, Kevan Mah, Chris Lee, and Chris Kong.
\newblock {Mitra-v2 technical report}, 2026.

\bibitem[Yarowsky(1995)]{yarowsky1995unsupervised}
David Yarowsky.
\newblock Unsupervised word sense disambiguation rivaling supervised methods.
\newblock In \emph{33rd annual meeting of the association for computational linguistics}, 1995.

\bibitem[Yu et~al.(2026)Yu, Picard, and Ahmed]{yu2026gitbo}
Rosen Yu, Cyril Picard, and Faez Ahmed.
\newblock {GIT-BO}: High-dimensional bayesian optimization with tabular foundation models.
\newblock In \emph{International Conference on Learning Representations}, 2026.

\bibitem[Zhang et~al.(2025{\natexlab{a}})Zhang, Ren, Yu, Yuan, Wang, Li, Wu, Mo, Mao, Hao, et~al.]{zhang2025limix}
Xingxuan Zhang, Gang Ren, Han Yu, Hao Yuan, Hui Wang, Jiansheng Li, Jiayun Wu, Lang Mo, Li~Mao, Mingchao Hao, et~al.
\newblock {Limix: Unleashing structured-data modeling capability for generalist intelligence}.
\newblock \emph{arXiv preprint arXiv:2509.03505}, 2025{\natexlab{a}}.

\bibitem[Zhang et~al.(2025{\natexlab{b}})Zhang, Maddix, Yin, Erickson, Ansari, Han, Zhang, Akoglu, Faloutsos, Mahoney, Hu, Rangwala, Karypis, and Wang]{zhang2025mitra}
Xiyuan Zhang, Danielle~C. Maddix, Junming Yin, Nick Erickson, Abdul~Fatir Ansari, Boran Han, Shuai Zhang, Leman Akoglu, Christos Faloutsos, Michael~W. Mahoney, Cuixiong Hu, Huzefa Rangwala, George Karypis, and Bernie Wang.
\newblock Mitra: Mixed synthetic priors for enhancing tabular foundation models.
\newblock In \emph{Advances in Neural Information Processing Systems}, 2025{\natexlab{b}}.

\end{thebibliography}
\bibliographystyle{iclr2027_conference}

\clearpage
\appendix
\crefalias{section}{appendix}
\crefalias{subsection}{appendix}
\crefalias{subsubsection}{appendix}
\counterwithin{proposition}{section}
\counterwithin{corollary}{section}
\counterwithin{remark}{section}
\counterwithin{definition}{section}

\section{Further theoretical results and proofs}
\label{app:proofs}

\subsection{Bayes-optimal predictive} \label{app:bop}

\begin{proposition}
\label{prop:nonseparable-minimizer}
Under our data model, the expected log loss over all conditional predictors
$q(y\mid\bm{x},\dn)$ is minimized, uniquely up to null sets of
$(\bm{X},\Dn)$, by the conditional distribution of $Y$ given
$(\bm{X},\Dn)$. Its density
$\ppd{y\mid\bm{x},\dn}$ averages $\task{y\mid\bm{x}}$
under the task posterior
\begin{equation}
    d\prior(p\mid\bm{x},\dn)
    \propto
    p_X(\bm{x})
    \prod_{i=1}^n p_X(\bm{x}_i)\,\task{y_i\mid\bm{x}_i}\,d\prior(p).
    \label{eq:nonseparable-ppd}
\end{equation}
This is equivalent to $\pi$ in~\cref{eq:ppd}.
Under separability, the marginal posterior of $p_{Y\mid X}$ is proportional to $\prod_{i=1}^n
    p_{Y\mid X}(y_i\mid\bm{x}_i)\,  d\prior_{Y\mid X}(p_{Y\mid X})$, and hence does not depend on the query feature $\bm{x}$. Therefore, the predictive equals
$\ppdbar{y\mid\bm{x},\dn}$ wherever the posteriors are defined.
\end{proposition}

\begin{proof}
Let $\featurecollectionandD=(\featurecollection,\Dn)$, where
$\featurecollection=(\bm{X}_1,\dots,\bm{X}_n,\bm{X})$ and $\responsecollection=(Y_1,\ldots,Y_n,Y)$, with realizations $\featurevalues$ and $\responsevalues=(y_1,\ldots,y_n,y)$.
For any conditional distribution $q$,
\begin{equation*}
    \E[-\log q(Y\mid\featurecollectionandD)]
    -
    \E[-\log \pi(Y\mid\featurecollectionandD)]
    =
    \E\!\left[
        \KL\!\left(
            \pi(\cdot\mid\featurecollectionandD)
            \,\Vert\,
            q(\cdot\mid\featurecollectionandD)
        \right)
    \right]
    \ge 0,
\end{equation*}
where $\pi(\cdot\mid\featurecollectionandD)$ is the conditional distribution
of $Y$ given $\featurecollectionandD$. Equality holds if and only if
$q(\cdot\mid\featurecollectionandD)
=\pi(\cdot\mid\featurecollectionandD)$ almost surely, proving uniqueness
up to null sets of the inputs.

By the independence of the rows given the task $p$, the joint density at $(\featurevalues,\responsevalues)$,
conditional on $p$, is
\[
    p_X(\bm{x})\,p(y\mid\bm{x})
    \prod_{i=1}^n p_X(\bm{x}_i)\,p(y_i\mid\bm{x}_i).
\]
Hence, conditioning on
$\bm{X}_i=\bm{x}_i$, $\bm{X}=\bm{x}$, and
$Y_i=y_i$, Bayes' rule gives
\[
    d\prior(p\mid\bm{x},\dn)
    \propto
    p_X(\bm{x})
    \prod_{i=1}^n p_X(\bm{x}_i)\,\task{y_i\mid\bm{x}_i}\,
    d\prior(p),
\]
which is~\cref{eq:nonseparable-ppd}. Moreover,
\[
    \pi(y\mid\bm{x},\dn)
    =
    \int
    \task{y\mid\bm{x}}\,
    d\prior(p\mid\bm{x},\dn),
\]
since, conditional on the task $p$ and $\featurecollection$, the distribution of $Y$
is $\task{\cdot\mid\bm{X}}$.

Under separability, the prior factorizes as
\[
    d\prior(p)
    =d\prior_X(p_X)\,d\prior_{Y\mid X}(p_{Y\mid X}).
\]
Thus, marginalizing the task posterior over $p_X$ gives
\begin{align*}
    d\prior(p_{Y\mid X}\mid\bm{x},\dn)
    &\propto
    \left[
        \int
        p_X(\bm{x})\prod_{i=1}^n p_X(\bm{x}_i)\,
        d\prior_X(p_X)
    \right]
    \prod_{i=1}^n
    p_{Y\mid X}(y_i\mid\bm{x}_i)\,
    d\prior_{Y\mid X}(p_{Y\mid X})
    \\
    &\propto
    \prod_{i=1}^n
    p_{Y\mid X}(y_i\mid\bm{x}_i)\,
    d\prior_{Y\mid X}(p_{Y\mid X}),
\end{align*}
because the bracketed factor does not depend on $p_{Y\mid X}$.
Consequently,
\[
    \pi(y\mid\bm{x},\dn)
    =
    \int
    p_{Y\mid X}(y\mid\bm{x})\,
    d\prior(p_{Y\mid X}\mid\dn)
    =
    \ppdbar{y\mid\bm{x},\dn}.
\]
\end{proof}

As a direct corollary, we can also derive the corresponding quantities when a set of unlabeled features is available.

\begin{corollary}
\label{cor:unlabeled-predictive}
Let $\bm{V}=(\bm{V}_1,\ldots,\bm{V}_k)$ be additional unlabeled feature
rows that, conditional on the task $p$, are iid draws from $p_X$ and are
independent of the context and query row. Let $\bm{v}$ denote their realizations. Then the expected log loss over
all conditional predictors $q(y\mid \bm{x},\bm{v},d_n)$ is minimized,
uniquely up to null sets, by
\begin{equation}
    \pi(y\mid \bm{x},\bm{v},d_n)
    =
    \int p(y\mid \bm{x})\,
    d\Pi(p\mid \bm{x},\bm{v},d_n),
\end{equation}
where
\begin{equation}
    d\Pi(p\mid \bm{x},\bm{v},d_n)
    \propto
    p_X(\bm{x})
    \prod_{j=1}^k p_X(\bm{v}_j)
    \prod_{i=1}^n
    p_X(\bm{x}_i)p(y_i\mid \bm{x}_i)\,
    d\Pi(p).
    \label{eq:unlabeled-predictive}
\end{equation}
Under a separable prior, the marginal posterior over $p_{Y\mid X}$ does
not depend on $\bm{x}$ or $\bm{v}$, and hence
\[
    \pi(y\mid \bm{x},\bm{v},d_n)
    =
    \bar{\pi}(y\mid \bm{x},d_n).
\]
\end{corollary}

\begin{proof}
Conditional on the task $p$, the labeled context rows, the unlabeled rows,
and the query row are independent. Hence, the joint density of the observed
features and context labels, conditional on $p$, is proportional to
\[
    p_X(\bm{x})
    \prod_{j=1}^k p_X(\bm{v}_j)
    \prod_{i=1}^n
    p_X(\bm{x}_i)p(y_i\mid\bm{x}_i).
\]
Bayes' rule therefore gives \cref{eq:unlabeled-predictive}. Conditional on
$p$ and $\bm{x}$, the query label has distribution $p(\cdot\mid\bm{x})$,
so its conditional distribution given $(\bm{x},\bm{v},d_n)$ is
\[
    \int p(y\mid\bm{x})\,
    d\Pi(p\mid\bm{x},\bm{v},d_n).
\]
By the same log-loss argument as in Proposition~A.1, this conditional
distribution is the unique Bayes-optimal predictor up to null sets.

If $\Pi=\Pi_X\otimes\Pi_{Y\mid X}$, marginalizing the posterior over
$p_X$ yields
\[
    d\Pi(p_{Y\mid X}\mid\bm{x},\bm{v},d_n)
    \propto
    \prod_{i=1}^n
    p_{Y\mid X}(y_i\mid\bm{x}_i)\,
    d\Pi_{Y\mid X}(p_{Y\mid X}),
\]
because all factors involving $\bm{x}$, $\bm{v}$, and the context feature
likelihood depend only on $p_X$. Thus the unlabeled rows and the query
feature do not update the posterior over the labeling mechanism, which
gives
$\pi(y\mid\bm{x},\bm{v},d_n)=\bar{\pi}(y\mid\bm{x},d_n)$.
\end{proof}
\subsection{Non-separability of SCM priors}

As discussed before, we should not expect independence to hold for SCM priors. The following shows that it already fails in simple cases such as a linear Gaussian SCM with a random label node. 

\begin{proposition}
\label{prop:scm-nonseparable}
Let $(B, J)$ be random, where $B$ is an acyclic $d \times d$ matrix and $J$ is a node, with $d \ge 2$. Conditional on $(B, J)$, let $\bm{Z} = B\bm{Z} + \bm{E}$ with $\bm{E} \sim \mathcal{N}(0, D)$ for a fixed positive diagonal matrix $D$, and let $Y := Z_J$ and $\bm{X} := \bm{Z}_{-J}$. If \,$\operatorname{Var}(Y \mid \bm{X}, B, J)$ is nonconstant under the prior on $(B, J)$, the induced prior is non-separable. This holds in particular when $D = I_d$, $J$ is uniform and independent of $B$, and $B$ has a nonzero edge with positive probability.
\end{proposition}

\begin{proof}
Let $\Sigma = (I-B)^{-1}D(I-B)^{-T}$. Acyclicity gives $\det(I-B) = 1$ and hence $\det\Sigma = \det D$. Since $\Sigma_{-J,-J}$ is positive definite, the block determinant formula and Gaussian conditioning give
\begin{align*}
    \det D
    = \det\Sigma
    &= \det(\Sigma_{-J,-J})\bigl(\Sigma_{JJ} - \Sigma_{J,-J}\Sigma_{-J,-J}^{-1}\Sigma_{-J,J}\bigr) \\
    &= \det(\Sigma_{-J,-J})\,\operatorname{Var}(Y \mid \bm{X}, B, J).
\end{align*}
The feature law determines $\det(\Sigma_{-J,-J})$, and the conditional label law determines $\operatorname{Var}(Y \mid \bm{X}, B, J)$. If the latter is nonconstant, these laws cannot be independent because the product is fixed.

Now consider the particular case $D = I_d$, with $J$ uniform and independent of $B$. The conditional variance is
\begin{align*}
    \operatorname{Var}(Y \mid \bm{X}, B, J)^{-1}
    &= (\Sigma^{-1})_{JJ}
     = 1 + \sum_{k: J \to k} B_{kJ}^2.
\end{align*}
Every finite acyclic graph has a node with no outgoing edges. If $J$ is such a node, the sum in the last display is empty and the conditional variance is $1$. If $J$ instead has a nonzero outgoing edge $J \to k$, it holds $B_{kJ} \ne 0$ and, thus, 
\begin{align*}
    \operatorname{Var}(Y \mid \bm{X}, B, J)
    &= \biggl(1 + \sum_{l: J \to l} B_{lJ}^2\biggr)^{-1}
     \le (1 + B_{kJ}^2)^{-1} < 1.
\end{align*}
On the event that $B$ has a nonzero edge, both types of node exist. Since $J$ is uniform and independent of $B$, each type is selected with conditional probability at least $1/d$. Since the event has positive probability, the conditional variance is nonconstant under the prior.
\end{proof}

\subsubsection{Proof of \cref{prop:mixture-nonseparability}} \label{proof:mixture-nonseparability}

\cref{prop:mixture-nonseparability} states that even for separable priors, mixtures are non-separable when both induced laws differ across regimes.

\begin{proof}
The marginals of $\prior$ are
\begin{align*}
    \prior_X
    &= (1 - \lambda)\prior_{X,0} + \lambda\prior_{X,1}, \\
    \prior_{Y\mid X}
    &= (1 - \lambda)\prior_{Y\mid X,0} + \lambda\prior_{Y\mid X,1},
\end{align*}
which implies
\begin{align*}
    \prior - \prior_X \otimes \prior_{Y\mid X}
    &= \lambda(1 - \lambda)
       (\prior_{X,1} - \prior_{X,0})
       \otimes (\prior_{Y\mid X,1} - \prior_{Y\mid X,0}).
\end{align*}
Since $\lambda(1-\lambda)>0$, the right-hand side vanishes iff $(\prior_{X,1}-\prior_{X,0})\otimes (\prior_{Y\mid X,1}-\prior_{Y\mid X,0})=0$. A tensor product of two finite signed measures is zero iff at least one factor is zero. Hence $\prior$ is separable iff $\prior_{X,0}=\prior_{X,1}$ or $\prior_{Y\mid X,0}=\prior_{Y\mid X,1}$.
\end{proof}

\subsection{Consistency and the finite-sample difference}

\subsubsection{Proof of \cref{cor:covariance-shift}} \label{proof:covariance-shift}

\begin{proof}
By~\cref{eq:nonseparable-ppd} and the definition of $\prior_n$, the posterior given the query has density $p_X(\bm{x}) / m(\bm{x} \mid \dn)$ relative to $\prior_n$, so $\ppd{y \mid \bm{x}, \dn} = \E_{\prior_n}[\task{y \mid \bm{x}}\, p_X(\bm{x})] / \E_{\prior_n}[p_X(\bm{x})]$, and $\E[AB] = \E[A]\E[B] + \operatorname{Cov}(A, B)$.
\end{proof}

\subsubsection{Proof of \cref{prop:nonseparable-consistency}} \label{proof:nonseparable-consistency}

\citet[Theorem 3.1]{nagler23statistical} shows for classification with iid rows that $\bar\pi$ converges to the conditional of the KL-optimal distribution in the prior's support. The proof is based on the task posterior $\prior_n$ formed from the joint likelihood, and it uses only that $\prior_n$ concentrates in Hellinger distance.
In the non-separable setting, the query feature itself changes the task posterior. The following proposition shows that concentration of the posterior over joint task laws is sufficient for convergence of both $\pi$ and $\bar\pi$, in both classification and regression problems. Sufficient KL-support and Hellinger-cover conditions for a unique KL-optimal density are given by \citet[Corollary 1]{deblasiwalker2013} and used in the PFN setting by \citet[Appendix A.2]{nagler23statistical}.

Recall that, for two densities $m$ and $p$ with respect to a common dominating measure $\mu$, the total variation and Hellinger distances are defined as
\begin{align*}
    \mathrm{TV}(m, p)
    &= \frac 12 \int |m - p|\, d\mu, \qquad
    H(m, p)
    = \left(\int (\sqrt{m} - \sqrt{p})^2\, d\mu\right)^{1/2}.
\end{align*}

\begin{proof}
    Since $\mathrm{TV}(m, p) \le H(m ,p)$ and $\mathrm{TV}(m, p) \le 1$, we get
\begin{align*}
    \delta_n 
    &= \int \mathrm{TV}(p, p^\ast)\, d\prior_n(p) \\
    &= \int_{\{H(p, p^\ast) \le \epsilon\}} \mathrm{TV}(p, p^\ast)\, d\prior_n(p)
        + \int_{\{H(p, p^\ast) > \epsilon\}} \mathrm{TV}(p, p^\ast)\, d\prior_n(p) \\
    &\le \epsilon
        + \prior_n\bigl(\{p : H(p, p^\ast) > \epsilon\}\bigr).
\end{align*}
Taking $n \to \infty$ followed by $\epsilon \to 0$ gives $\limsup_{n \to \infty} \delta_n = 0$ almost surely.

For convergence of the two predictives $\pi$ and $\bar\pi$, we first show that the integrated total variation of the conditional distributions is bounded by the total variation of the joint distributions. If $m$ and $p$ are joint densities on $\Ycal \times \Xcal$, the product of the label and feature spaces, with feature marginals $m_X$ and $p_X$, then
\begin{align}
    &\quad \; \int \mathrm{TV}\bigl(m(\cdot \mid \bm{x}), p(\cdot \mid \bm{x})\bigr)
        p_X(\bm{x})\, d\bm{x} \notag \\
    &= \frac12\iint
        \bigl|m(y \mid \bm{x})p_X(\bm{x}) - p(y, \bm{x})\bigr|\, dy\, d\bm{x} \notag\\
    &\le \frac12\iint
        \bigl|m(y \mid \bm{x})p_X(\bm{x}) - m(y, \bm{x})\bigr|\, dy\, d\bm{x}  + \frac12\iint
        \bigl|m(y, \bm{x}) - p(y, \bm{x})\bigr|\, dy\, d\bm{x} \notag\\
    &= \mathrm{TV}(m_X, p_X) + \mathrm{TV}(m, p) \notag\\
    &\le 2\,\mathrm{TV}(m, p).
    \label{eq:conditional-tv}
\end{align}

Now let $m_n(y, \bm{x}) = \int p(y, \bm{x})\, d\prior_n(p)$ be the posterior predictive joint density of one further row, with feature marginal $m_{n,X}$. Because the rows are conditionally iid, $\ppd{y \mid \bm{x}, \Dn} = m_n(y, \bm{x})/m_{n,X}(\bm{x})$. Applying \cref{eq:conditional-tv} to $m_n$, gives
\begin{align*}
    \int \mathrm{TV}\bigl(\ppd{\cdot \mid \bm{x}, \Dn}, p^\ast(\cdot \mid \bm{x})\bigr)
        p^\ast_X(\bm{x})\, d\bm{x}
    &\le 2\,\mathrm{TV}(m_n, p^\ast) \\
    &= 2\,\mathrm{TV}\left(\int p(y, \bm{x})\, d\prior_n(p), \, p^\ast\right) \\
    &\le 2 \int \mathrm{TV}(p, p^\ast)\, d\prior_n(p) \\
    &= 2\delta_n,
\end{align*}
The second inequality uses convexity of total variation. This proves the first bound.

For the second bound, recall that $\ppdbar{\cdot \mid \bm{x}, \Dn}$ is the posterior mixture of $p(\cdot \mid \bm{x})$. Convexity of total variation applied to this mixture gives, for every $\bm{x}$,
\begin{align*}
    \mathrm{TV}\bigl(\ppdbar{\cdot \mid \bm{x}, \Dn}, p^\ast(\cdot \mid \bm{x})\bigr)
    &\le \int \mathrm{TV}\bigl(p(\cdot \mid \bm{x}), p^\ast(\cdot \mid \bm{x})\bigr)
        \, d\prior_n(p).
\end{align*}
Integrating over $p^\ast_X(\bm{x})$, changing the order of integration, and \cref{eq:conditional-tv} give
\begin{align*}
    \int \mathrm{TV}\bigl(\ppdbar{\cdot \mid \bm{x}, \Dn}, p^\ast(\cdot \mid \bm{x})\bigr)
        p^\ast_X(\bm{x})\, d\bm{x}
    &\le \int \biggl[\int
        \mathrm{TV}\bigl(p(\cdot \mid \bm{x}), p^\ast(\cdot \mid \bm{x})\bigr)
        p^\ast_X(\bm{x})\, d\bm{x}\biggr] d\prior_n(p) \\
    &\le 2\int \mathrm{TV}(p, p^\ast)\, d\prior_n(p) \\
    &= 2\delta_n.
\end{align*}
\end{proof}

\subsection{Misspecification}

\subsubsection{Proof of \cref{prop:misspecified-limit}} \label{proof:misspecified-limit}

\begin{proof}
(i) Under separability, the support of the induced pairs $(p_X, p_{Y\mid X})$ is the Cartesian product $ \mathrm{supp}(\Pi_X) \times \mathrm{supp}{(\Pi_{Y \mid X})}$. For every supported mechanism $p$, it then contains the pair $(p^\ast_X, p^\ast_{Y\mid X})$. Joint optimality of $p^\ast$ and the KL chain rule therefore give
\begin{align*}
    \KL(p_{0,X}\,\Vert\,p^\ast_X)+\KL^{Y\mid X}(p_0\,\Vert\,p^\ast)
    &\le \KL(p_{0,X}\,\Vert\,p^\ast_X)+\KL^{Y\mid X}(p_0\,\Vert\,p).
\end{align*}
The feature divergence  $\KL(p_{0,X}\,\Vert\,p^\ast_X)$ is finite by the assumption on $p^\ast$. Canceling it and taking the infimum over $p$ proves (i).

(ii) We construct an explicit model to verify the claim. Suppose $X$ and $Y$ are independent $\mathcal N(0, 1)$ variables under $p_0$, and let $t>0$ be arbitrary. Consider a prior with positive mass on only the following two task densities $p^{(A)}$ and $p^{(B)}$:
\begin{align*}
    p^{(A)}:\quad & X\sim\mathcal N(\sqrt{2(t+1)},1), & Y\mid X &\sim\mathcal N(0,1),\\
    p^{(B)}:\quad & X\sim\mathcal N(0,1), & Y\mid X &\sim\mathcal N(\sqrt{2t},1).
\end{align*}
By the standard formula for KL divergence between Gaussians,
\begin{align*}
    \KL( \mathcal N(\mu_0, \sigma_0^2)\,\Vert\,\mathcal N(\mu_1, \sigma_1^2)) = \frac12\left(\frac{(\mu_0-\mu_1)^2}{\sigma_1^2} + \frac{\sigma_0^2}{\sigma_1^2} - 1 - \log\left(\frac{\sigma_0^2}{\sigma_1^2}\right)\right), 
\end{align*}
the feature and conditional KL divergences are 
\begin{align*}
    \bigl(\KL(p_{0,X}\,\Vert\,p^{(A)}_X),\KL^{Y\mid X}(p_0\,\Vert\,p^{(A)})\bigr)
    &=(t+1,0),\\
    \bigl(\KL(p_{0,X}\,\Vert\,p^{(B)}_X),\KL^{Y\mid X}(p_0\,\Vert\,p^{(B)})\bigr)
    &=(0,t).
\end{align*}
Thus, $p^{(B)}$ uniquely minimizes the joint KL divergence, while its conditional KL exceeds the supported minimum by $t$. The prior contains both the true feature law and the true conditional label law, but not their pairing $p_0$. The prior is therefore non-separable, and the penalty is unbounded as $t \to \infty$.
\end{proof}

\subsection{Martingale checks}
\label{app:martingale}

Fix a label event $A\subseteq\Ycal$, a query $\bm{x}$ and a future feature $\bm{x}_{n+1}$. The diagnostic draws $Y_{n+1}\sim q(\cdot\mid\bm{x}_{n+1},\dn)$, appends $(Y_{n+1},\bm{x}_{n+1})$ to the context, and compares the expected updated prediction at $\bm{x}$ with its initial value. For the exact target, the tower property gives
\begin{equation*}
    \E_{Y_{n+1}\sim\pi(\cdot\mid\bm{x}_{n+1},\bm{x},\dn)}
    \bigl[\pi(A\mid\bm{x},(\dn,(Y_{n+1},\bm{x}_{n+1})))\bigr]
    =\pi(A\mid\bm{x},\bm{x}_{n+1},\dn).
    \label{eq:martingale-drift}
\end{equation*}
Unlike the diagnostic in \cref{eq:martingale-check}, this identity conditions the label sampler on the tracked query and retains the additional feature evidence on the right-hand side. The feature arguments represent distinct rows even when their values coincide. Sampling the entire future row from its joint posterior predictive conditional on $(\bm{x},\dn)$ recovers zero drift by another application of the tower property.

\paragraph{Cumulative drift in the experiments.}
For $\breve\pi\in\{\pi,\bar\pi\}$, fix an initial context $\dn$, a query $\bm{x}$, and a repeated future feature $\bm{x}^\star$. Starting from $D_n=\dn$, generate additional
rows sequentially according to
\[
    Y_{n+k+1}\mid D_{n+k}
    \sim \breve\pi(\cdot\mid\bm{x}^\star,D_{n+k}),
    \qquad
    D_{n+k+1}
    = \bigl(D_{n+k},(Y_{n+k+1},\bm{x}^\star)\bigr).
\]
Each appended row updates the context through its joint likelihood. For the label event $A=\{1\}$, the cumulative expected drift after $K$ appended rows is
\[
    \Delta_K
    =
    \E\!\left[
        \breve\pi(1\mid\bm{x},D_{n+K})
        \,\middle|\,\dn,\bm{x},\bm{x}^\star
    \right]
    - \breve\pi(1\mid\bm{x},\dn).
\]
where the expectation is over future labels generated from $\breve\pi$. The feature $\bm{x}^\star$ is either the query itself or one context feature selected uniformly once per context-query pair and held fixed throughout the rollout. The cumulative-drift panel of \cref{fig:martingale} reports the root mean square of these differences across $M=2000\cdot 8=16000$ context-query pairs:
\[
    \operatorname{RMS}_{\breve\pi}(K)
    =
    \left[
        \frac{1}{M}\sum_{i=1}^{M}
        \bigl(\Delta_K^{(i)}\bigr)^2
    \right]^{1/2},
\]
where $\Delta_K^{(i)}$ denotes $\Delta_K$ evaluated for the $i$th context--query pair and its selected future feature. Thus, the expectation over future labels is taken before squaring, and the square root returns the statistic to probability units. Squaring prevents positive and negative expected drifts from cancelling across pairs, while giving larger drifts more weight. The statistic measures the cumulative expected change from the initial prediction; it does not sum squared one-step changes.

For binary labels at a repeated feature, the predictive after $k$ appended rows depends on their labels only through the number of class-1 outcomes. We therefore propagate the probabilities of the $k+1$ possible counts, integrating over all future label sequences without Monte Carlo sampling.

\section{Remarks and further details}
\label{app:remarks}

\paragraph{Details on the amortization gap}
\label{rem:amortization}
Following \citet{hulskamp2026on}, we use \emph{amortization gap} for the total excess of a PFN over the Bayes-optimal predictor, covering both the limits of the model class and those of finite pretraining. Pointwise, it is $\gap(\bm{x},\dn)=\KL\bigl(\ppd{\cdot\mid\bm{x},\dn}\,\Vert\,\pfn{\cdot\mid\bm{x},\dn}\bigr)$. During pretraining, the context size $N\sim\sizeprior$ is drawn independently of the task $p\sim\prior$. For the random-size context $\DN=(Y_i,\bm{X}_i)_{i=1}^N$, write $\mathcal{L}(\btheta)=-\E[\log\pfn{Y\mid\bm{X},\DN}]$ and $\mathcal{L}^\ast=-\E[\log\ppd{Y\mid\bm{X},\DN}]$, where $\E=\E_{\sizeprior}\E_{\prior}$ averages over context sizes, tasks, and data. Since $\pi$ minimizes the pretraining loss (\cref{prop:nonseparable-minimizer}), its expectation equals $\mathcal{L}(\btheta)-\mathcal{L}^\ast$ with $\mathcal{L}^\ast=-\E_{\sizeprior}\E_{\prior}[\log\ppd{Y\mid\bm{X},\DN}]$. Both statements require $\pi$ as the reference, though under a non-separable prior, a reference computed from the context posterior alone yields $\bar\pi$ instead. An example would be MCMC draws from $\prior(\cdot\mid\dn)$ that are not reweighted by the query's feature likelihood $p_X(\bm{x})$. The discrepancy to this reference then decomposes as
\begin{equation*}
    \pfn{y\mid\bm{x},\dn}-\ppdbar{y\mid\bm{x},\dn}
    = \bigl(\pfn{y\mid\bm{x},\dn}-\ppd{y\mid\bm{x},\dn}\bigr) + \predgap,
\end{equation*}
so the predictive gap $\predgapnoarg$ (\cref{cor:covariance-shift}) is added on top of the approximation error, and any divergence or loss difference computed against $\bar\pi$ mixes the two.
On real data, labels do not follow $\pi$, so even with the correct reference the loss difference is no longer a divergence and can be negative (cf.~\citealp{hulskamp2026on}).

\paragraph{Counterexample for martingale zero-drift under non-separable priors} Consider two equally likely tasks on $X\in\{0,1,2\}$ with feature probabilities $(0.5,0.4,0.1)$ and $(0.5,0.1,0.4)$, and class 1 probabilities at $X=0$ of $0.2$ and $0.8$, respectively. At $X=1,2$, let both tasks assign class 1 probability $0.5$. This prior is non-separable, but fixing the query and all future features at zero yields a martingale: their common feature likelihood $0.5$ cancels from every update, leaving ordinary Bayesian label updates. Non-separability therefore permits drift without requiring it at every query.

The two terms can also cancel while both are nonzero. With no initial context, take two equally likely tasks on $X\in\{0,1\}$ with $p_X(0)$ equal to $0.8$ and $0.2$, and deterministic labels $Y=1$ and $Y=0$, respectively. At query and future feature $x=0$, the initial class 1 probability is $4/5$ and conditioning on both feature rows gives $16/17$. The sampled label identifies the task, so the feature-update and sampling terms are $12/85$ and $-12/85$: the net drift is zero.

\paragraph{Transductive channels in current benchmarks and models}
\label{rem:transductive}
Currently, the handling of test data differs substantially between different benchmark suites as well as between different models. Below we summarize important aspects with regard to possible transductive channels in some of the most popular benchmarks and models contained in them:  
\begin{itemize}
    \item \emph{TabArena} \citep{erickson2026tabarena} passes the full test fold to the model. It has no written rule on transductive use. Its only check is a smoke test that compares the prediction for a single row alone and within the batch, and methods can be exempted from it.
    \item \emph{TALENT} \citep{liu2025talent} processes test minibatches in its neural-model evaluation routine, so a prediction can depend on the composition of the batch.
    \item \emph{ScoringBench} \citep{landsgesell2026scoringbench} passes all test features jointly to its predictive-distribution interface.
    \item \emph{BeyondArena} \citep{purucker2026beyond} explicitly documents transductive preprocessing for label-per-group tasks. It computes feature summaries within each test group and attaches them to that group's rows \citep[Appendix~C.2]{purucker2026beyond}. This reflects tasks whose prediction unit is naturally a group and does not by itself make comparisons unfair.
\end{itemize}
On the model side, the commonly used attention masks, for example in TabPFN and TabICL, prevent direct attention between test rows. TabICL also excludes test rows from the shared context of its column embedding by default and includes them only with \texttt{embed\_with\_test=True}. While these restrictions blocks large interactions, they do not make the complete pipeline independent of the other test rows:
\begin{itemize}
    \item The standard forward path of TabPFN-v2.5 computes constant-feature masks and feature-group normalization parameters from the combined training and test input.
    \item TALENT's bundled LimiX implementation \citep{zhang2025limix} performs feature filtering, categorical-type inference, and ordinal encoding on the combined training and test features.
    \item TALENT's default SAINT configuration enables intersample attention, so its predictions explicitly depend on the other rows in the test minibatch \citep{somepalli2021saint}.
    \item Proprietary models served through an API, such as TabPFN-Thinking, receive the whole test fold server-side and can only be probed as a black box for batch invariance.
\end{itemize}
In total, an evaluation protocol should specify the prediction unit, which unlabeled observations are available and when, and how each method uses them, including any dependence on the composition of the test batch. Per-row prediction calls enforce invariance only if preprocessing and retained state have not already used other test rows. When transductive inference is permitted, all methods should receive the same information, although they need not exploit it identically.

\section{Experimental details} \label{app:experiments}

\subsection{Ordered-Cluster Prior} \label{app:exp_ordered_cluster_prior}

\paragraph{Prior} We introduce the \textit{ordered-cluster prior} that lets us study our theoretical claims empirically, where we pretrain TabICL models from scratch on this custom prior. It is a finite non-separable classification prior over task distributions $p_{a,b}$, where the centers $a<b$ are chosen uniformly from the $36$ ordered pairs on the grid $\{-4,\ldots,4\}$. For each row, the class is drawn uniformly, with class $0$ generating features from $\mathcal{N}(a,0.5^2)$ and class $1$ from $\mathcal{N}(b,0.5^2)$. Because the same cluster centers determine both the feature distribution and the labeling mechanism, an unlabeled query feature can provide information about its label. More precisely, conditioning on a query $\bm{x}$ changes the label prediction exactly when its likelihood under the posterior tasks is correlated with their class probability at $\bm{x}$ (see \cref{eq:covariance-shift}). This effect is particularly pronounced when the context contains only one class. In this case, the corresponding center is observed, while the other remains uncertain, so the query location helps identify the missing center. This is also the effect illustrated in \cref{fig:toy_query_evidence}. The finite task space makes both the query-updated predictive $\pi$ and the context-updated predictive $\bar{\pi}$ exactly enumerable, providing a controlled setting in which to determine which predictive standard pretraining recovers.

\cref{fig:approximation_target_exp}: We train a $2.6$M parameter version of TabICL on the non-separable ordered-cluster prior using two different query-labeling schemes: standard pretraining and a decoupled-query control, denoted \texttt{decoupled}. In both schemes, we sample a latent task $p$ and use it to generate the context $\dn$ and query features $\bm{x}$. Under standard pretraining, the query labels are generated from the same task $p$, so their conditional distribution is the query-updated predictive $\pi(y\mid\bm{x},\dn)$. Under the decoupled setting, we instead draw a new task $p'$ from the context posterior $\prior(\cdot\mid\dn)$, without conditioning on the query features, and generate each query label from $p'(\cdot\mid\bm{x})$. Conditional on the context, this decouples the query feature from the task used to generate its label. The resulting marginal label distribution is therefore the context-updated predictive $\bar{\pi}$ rather than $\pi(y\mid\bm{x},\dn)$. The two schemes thus receive identical contexts and query features but have different Bayes-optimal targets: $\pi$ for standard pretraining and $\bar{\pi}$ for the decoupled setting. At context size $n=2$, the figure plots the network difference $q_1-\bar{\pi}_1$ against the exact target difference $\pi_1-\bar{\pi}_1$ (predictives for class 1), averaged over three training seeds after $50\,000$ optimizer steps. For each query both coordinates are differences against the same baseline: the exact difference between the two predictives of~\cref{eq:ppd}, $\ppd{1 \mid \bm{x}, \dn} - \ppdbar{1 \mid \bm{x}, \dn}$ (horizontal), and the network's own difference $\pfn{1 \mid \bm{x}, \dn} - \ppdbar{1 \mid \bm{x}, \dn}$ (vertical); a network at $\pi$ therefore lies on the dashed diagonal and one at $\bar\pi$ on the dotted horizontal. Queries are grouped into $20$ bins of equal width in the exact difference, with the bin edges running from its $0.5$th to its $99.5$th percentile; bins holding fewer than $20$ queries are dropped, and each marker averages both coordinates over its bin, with the band giving one standard deviation across the three training seeds. Standard pretraining closely follows the identity line, indicating recovery of $\pi$, whereas the decoupled setting remains near zero, indicating recovery of $\bar{\pi}$. This means that pretraining on a non-separable prior actually approximates the query-updated predictive $\pi$ of \cref{eq:ppd} and in this experiment, using the wrong reference can make the apparent approximation error roughly $95$ times larger than the true one.

The two panels differ for a simple reason. In a single-class context, all labeled examples come from one cluster, so the location of the other cluster remains uncertain. The query feature is then valuable unlabeled evidence about that unseen cluster and can substantially change the prediction, producing a large difference between $\pi$ and $\bar{\pi}$. In a two-class context, examples from both clusters already reveal both centers and the query adds little new information, so the target difference, and hence the plotted range, is much smaller ($y$-axis).

\paragraph{Training details} All networks are TabICL architectures \citep{qu2025tabicl} trained from scratch on the prior described above and consist of a column embedding, row interaction, and an in-context transformer. Two sizes, 2.6M parameters (embedding 64, blocks 2/2/4) and 8.1M (embedding 96, blocks 2/2/6). We use the Muon optimizer \citep{jordan2024muon} with standard cross-entropy loss, learning rate $8 \cdot 10^{-4}$, cosine schedule with restarts and 2\% warmup, gradient clipping at 10, a batch size of 64 datasets, and 50\,000 optimizer steps, with one pass over freshly simulated datasets.

\paragraph{Semi-supervised gains on the ordered-cluster prior}\label{app:semi-ordered-cluster}

On the ordered-cluster prior, the semi-supervised target $\ppd{y\mid\bm{x},\bm{v},\dn}$ of \cref{cor:unlabeled-predictive} can be computed exactly, which lets us check whether a trained PFN realizes the gain from unlabeled rows. We insert $k$ unlabeled rows $\bm{v}$ from the generating task between the context and the queries. We pretrain TabICL with access to them either through the in-context transformer or through the column embedding (\texttt{embed\_with\_test}), and train a control with $\bm{v}$ masked. At $n=2$, the unlabeled rows lower the Bayes-optimal loss by $11\%$, $34\%$ and $38\%$ for $k=1,8,32$. For $k\ge8$ and $n\le4$, the expected divergence $\KL(\ppd{\cdot\mid\bm{x},\bm{v},\dn}\,\Vert\,q_{\btheta})$ of the in-context network stays below $2\%$ of this gain, and that of the column-embedding network below $5\%$. The masked control stays at $\pi$. Both networks approximate their own targets closely, so their loss difference, which matches the gain to within $0.2\%$, reflects the larger information set rather than better approximation. For a more thorough experiment regarding the use of additional unlabeled queries, refer to \cref{app:semitabicl}.

\subsection{Feature effects under a prior with constant class probability}\label{app:feature-effects}

\paragraph{Prior} \Cref{fig:interpretability} is computed exactly on a finite non-separable classification prior with two tasks $p_s$, $s\in\{-1,1\}$, each with prior probability $1/2$. Under $p_s$ the three features are independent. $X_1$ and $X_2$ take values in $\{-2,\dots,2\}$ with $p_s(X_1=t)\propto\exp\{-(t-0.75s)^2/2\}$ and $p_s(X_2=t)\propto\exp\{-(t-0.5s)^2/2\}$, and $X_3$ is uniform on $\{-1,1\}$. The label does not depend on the features, $Y\mid\bm{X}\sim\mathrm{Bernoulli}(c_s)$ with $c_{-1}=0.2$ and $c_1=0.8$. Every task therefore has a constant conditional label distribution, whereas $X_1$ and $X_2$ carry information about the task and hence about $c_s$. $X_3$ has the same distribution under both tasks and serves as a negative control. All $50$ feature combinations have positive probability under both tasks, so every input evaluated by PD, ICE and SHAP lies in the support.

\paragraph{Predictives} Let $w=\prior(s=1\mid\dn)$. By the symmetry of the grid, the normalizing constants cancel and $\log\{r_{p_1}(\bm{x})/r_{p_{-1}}(\bm{x})\}=1.5x_1+x_2$, so that
\begin{equation*}
    \ppd{1\mid\bm{x},\dn}=0.2+0.6\,\sigma\bigl(\operatorname{logit}w+1.5x_1+x_2\bigr),
    \qquad
    \ppdbar{1\mid\bm{x},\dn}=0.2+0.6\,w,
\end{equation*}
where $\sigma$ is the logistic function. The figure uses the context $\dn$ with $n=2$ rows, $\bm{x}_1=(-1,1,-1)$ with $y_1=0$ and $\bm{x}_2=(1,-1,1)$ with $y_2=1$. Both the feature and the label likelihood ratios of this context equal one, so $w=1/2$ and $\ppdbar{1\mid\bm{x},\dn}=0.5$ for all $\bm{x}$. The context was chosen for illustration, not drawn from the prior.

\paragraph{Explanations} All methods are applied to $f(\bm{x})=\ppd{1\mid\bm{x},\dn}$ and to $\ppdbar{1\mid\bm{x},\dn}$ with the same background distribution, the feature predictive $B=m(\cdot\mid\dn)$. In this case it is the equal-weight mixture of the two feature distributions. All quantities are finite sums over the grid, and no network is trained. The partial dependence of $X_1$ is $\mathrm{PD}_1(t)=\E_{B}[f(t,X_2,X_3)]$, with $(X_2,X_3)$ drawn from their joint distribution under $B$. The ICE curves are $t\mapsto f(t,x_2,x_3)$ for $x_2\in\{-2,\dots,2\}$ \citep{goldstein2015peeking}. Since $f$ does not depend on $x_3$, each $x_2$ yields a single curve. Exact Shapley values are obtained by enumerating all $2^3$ coalitions at each of the $50$ grid points for two value functions. The marginal game uses $v_{\bm{x}}(S)=\E_{B}[f(\bm{x}_S,\bm{X}_{\bar S})]$, where $\bm{X}_{\bar S}$ follows its joint marginal under $B$. The conditional game uses $v_{\bm{x}}(S)=\E_{B}[f(\bm{X})\mid\bm{X}_S=\bm{x}_S]$ \citep{janzing2020feature}. The right panel reports $\E_{\bm{X}\sim B}\lvert\phi_j(\bm{X})\rvert$, which for $\pi$ equals $0.134$ and $0.073$ for $X_1$ and $X_2$ under the marginal game and $0.133$ and $0.084$ under the conditional game. Because $\bar\pi$ is constant, its PD and ICE curves are constant and its Shapley values vanish under both games. The attribution of $X_3$ under $\pi$ is zero in both games as well, since $f$ does not depend on $x_3$ and $X_3$ is independent of $(X_1,X_2)$ under $B$.

\subsection{SemiTabICL: Architectural details and evaluation}
\label{app:semitabicl}

\paragraph{Architecture}
SemiTabICL uses the three-stage TabICL architecture with a column embedding (2 blocks, embedding dimension 128, 128 inducing points), a row interaction (2 blocks, 3 CLS tokens) and an in-context transformer (6 blocks, width $3\cdot128=384$), each with 8 attention heads, for 8.5M parameters in total. Every dataset is laid out as $[\,\dn \mid \bm{v} \mid \text{queries}\,]$ with $n$ labeled, $k$ unlabeled and $m$ query rows. Wherever TabICL attends to the training rows, i.e.\ in the inducing-point attention of the column embedding and in the in-context transformer, the attention context consists of the first $n+k$ rows, so labeled and unlabeled rows attend to each other and queries attend to both. Unlabeled rows carry a single learned vector in place of the label embedding. The loss is the cross-entropy on the query rows only. The baseline TabICL is identical except for this vector and never sees unlabeled rows.

\paragraph{Prior}
Both models are trained on the unchanged graph-SCM prior of TabICLv2 \citep{qu2026tabiclv2} (random DAGs with 2--32 nodes, 1--100 features, up to 10 classes, predictability filters). The unlabeled rows are further samples of the same SCM whose labels are dropped. For SemiTabICL, every dataset has such a pool: the labeled rows make up 30\% of a sequence length drawn log-uniformly from $[14,334]$, so $n\in[4,100]$ and $m\approx2.3\,n$, and the pool has $k=\mathrm{round}(r\,n)$ rows with $r$ log-uniform on $[1,9]$. The number of classes is capped at $\max(2,\lfloor\min(n,m)/2\rfloor)$. The median synthetic dataset has $n=20$, $k=59$ and $m=47$ rows. The baseline is trained on the same distribution of $(n,m)$ without a pool.

\paragraph{Training}
We run pretraining for $30\,000$ optimizer steps, 1\% warmup, with bfloat16 precision, and batches of 64 datasets in groups of 16 that share their sequence length and pool size. Each model trains in about two hours on four NVIDIA H200 GPUs.

\paragraph{Compared methods}
All methods use an ensemble of 8 members and fit preprocessing on the labeled rows only, and all see identical splits. We compare SemiTabICL with its additional unlabeld context rows (\emph{SemiTabICL}), SemiTabICL with a the unlabeled context whose columns are permuted independently, which keeps the marginals but destroys the joint structure (\emph{shuffled pool}), the baseline TabICL on the labeled rows (\emph{TabICL}), and self-training (ST) with scikit-learn's \texttt{SelfTrainingClassifier} \citep{pedregosa2011scikit,yarowsky1995unsupervised} (confidence threshold $0.75$, at most 10 rounds) around TabICL and around a gradient-boosted tree classifier (with default settings). As an upper bound, the oracle gives TabICL the pool together with its labels. The purpose of the shuffled unlabeled context is to assess to what extent possible gains from the additional unlabeled context come from only refining the marginal distributions per column vs. the model actually learning to leverage the full distribution of the additional data.

\paragraph{semi-TALENT}
We use the 150 classification datasets of TALENT \citep{liu2025talent} with at most 100 features and 10 classes, excluding the subset reserved for validation during pretraining. From the official training split, we draw a stratified labeled sample of $n\in\{16,32,64,96\}$ rows and a pool of $k=r\,n$ further rows with $r\in\{2,4,8,10\}$, where pools of one labeled sample are nested. Each budget is repeated five times and evaluated on the official test split. On average over datasets, TabICL reaches $72.6\%$ accuracy, SemiTabICL $73.3\%$ and the oracle $77.9\%$. 

\paragraph{semi-UCI}
Following \citet{guo2010extensive}, we use their 26 UCI datasets with raw features and repeat a 4-fold stratified cross-validation 10 times. In each outer training fold, 10\% of the rows are labeled and the remainder forms the pool, so that $k=9n$, and the held-out fold is the test set. Two datasets with more than 10 classes (letter, vowel) and one with a preprocessing failure (sick) are excluded, leaving 23 datasets. Grouped by the categories of \citet{schoelkopf2012causal}, the structural gain is larger on the 11 anticausal or confounded datasets ($+0.70$ pp) than on the three causal ones ($+0.30$ pp), in line with their hypothesis. With only three causal datasets, however, the difference is not significant (exact permutation test, $p=0.42$).

\begin{table}[h]
    \centering
    \caption{Paired comparisons of SemiTabICL and its controls. $\Delta$acc is the difference in accuracy (percentage points) and $\Delta$NLL the reduction in log loss (nats), both positive when the first method is better. For each dataset, differences are averaged over all its splits; for semi-TALENT, these pool the full $4\times4$ grid of $n\in\{16,32,64,96\}$ and $r\in\{2,4,8,10\}$ with five repetitions each. Reported values are means of these per-dataset differences. For semi-TALENT, win is the share of the 150 datasets with a positive difference, and $p$ is a two-sided Wilcoxon signed-rank test of whether the per-dataset differences are centred at zero.}
    \label{tab:semitabicl}
    \small
    \begin{tabular}{l rrr rrr rr}
        \toprule
        & \multicolumn{6}{c}{semi-TALENT } & \multicolumn{2}{c}{semi-UCI } \\
        \cmidrule(lr){2-7}\cmidrule(lr){8-9}
        Comparison & $\Delta$acc & win & $p$ & $\Delta$NLL & win & $p$ & $\Delta$acc & $\Delta$NLL \\
        \midrule
        SemiTabICL vs.\ TabICL           & $+0.63$ & $0.81$ & $5\mathrm{e}{-16}$ & $+0.015$ & $0.85$ & $6\mathrm{e}{-19}$ & $+0.48$ & $+0.014$ \\
        SemiTabICL vs.\ shuffled     & $+0.66$ & $0.67$ & $1\mathrm{e}{-10}$ & $+0.017$ & $0.65$ & $2\mathrm{e}{-7}$  & $+0.47$ & $+0.017$ \\
        Shuffled vs.\ TabICL        & $-0.03$ & $0.59$ & $0.56$             & $-0.001$ & $0.63$ & $0.06$             & $+0.01$ & $-0.003$ \\
        TabICL ST vs.\ TabICL & $-0.33$ & $0.41$ & $3\mathrm{e}{-4}$  & $-0.440$ & $0.03$ & $6\mathrm{e}{-26}$ & $-0.53$ & $-0.393$ \\
        GBDT ST vs.\ GBDT     & $+0.45$ & $0.77$ & $3\mathrm{e}{-14}$ & $-0.577$ & $0.01$ & $2\mathrm{e}{-26}$ & $+0.46$ & $-0.452$ \\
        Oracle vs.\ TabICL               & $+5.18$ & $1.00$ & $2\mathrm{e}{-26}$ & $+0.125$ & $1.00$ & $2\mathrm{e}{-26}$ & $+5.33$ & $+0.120$ \\
        \bottomrule
    \end{tabular}
\end{table}

\subsection{Software and computational environment}
\label{app:software_comp_env}

Our experiments were implemented on top of the official TabICL software repository \citep{qu2025tabicl,qu2026tabiclv2} based on \texttt{pytorch} \citep{paszke_2019_PyTorchImperative}. TabICL pretraining runs were carried out on NVIDIA H200 GPUs with a full pretraining run with up to 50k optimizer steps taking less than 3 hours of wall-clock time for both the 2.6M parameter model as well as the 8.1M and 8.5M parameter model.

\section{Related work}

\paragraph{Semi-supervised learning (in ICL)} \citet{schoelkopf2012causal} point out the importance of analyzing the causal mechanisms underlying datasets for supervised learning and highlight under what conditions additional unlabeled data should prove useful with semi-supervised approaches. \citet{li2026and} show, from an expressivity perspective, that linear attention needs depth to exploit unlabeled in-context rows under a binary Gaussian-mixture prior, which is a special case of what we call a non-separable prior. Though they mention TabPFN as an application, the question studied is orthogonal to ours: we ask whether a network pretrained on a non-separable prior natively exploits additional unlabeled rows, as its objective says it should, whereas they wrap an off-the-shelf TabPFN in an external pseudo-labeling loop and show that self-training can improve predictive performance. \citet{swelam2025does} study whether the internal representations of TabPFN contain causal information and how well these can be repurposed for causal discovery by training an additional decoder. Their positive findings provide adjacent evidence that predictive pretraining yields feature representations that not only identify the latent mechanism $p(y \mid \bm{x})$ but also $p(\bm{x})$ as required for the query-updated predictive of \cref{eq:nonseparable-ppd}.

\paragraph{Bayesian targets and direct reference comparisons.}\citet{muller_2022_TransformersCana} introduce PFNs and show that prior-data NLL equals the expected KL to the PPD up to the Bayes-optimal loss, so held-out NLL measures the amortization gap whenever that loss is known, which is measured for a GP with fixed hyperparameters and a Gaussian target whose mean is linear in the context labels. In these settings the PFN approaches the exact NLL with more pretraining. Furthermore, \citet{muller2024bayes} extends the analysis by enumerating finite libraries of one-dimensional step, sine, and linear functions and show on selected datasets that predictive means follow the exact posterior, including non-obvious generalizations. This setup also separates failures of the posterior from failures of the network, such as an inability to count repeated observations. For intractable priors, the existing comparisons rank a PFN against NUTS, variational inference, or MLE-II \citep{muller_2022_TransformersCana}, primarily designed as speed comparisons and are very informative as such, but a Monte Carlo reference whose NLL never falls below the PFN's cannot locate the Bayes-optimal loss, so the gap under a general prior has not been measured. The later literature has largely turned to predictive performance on real benchmarks \citep{hollmann2025tabpfnv2,qu2025tabicl,grinsztajn2025tabpfnv25,qu2026tabiclv2}, which establishes that PFNs predict well but not what they compute.

\paragraph{Neural processes} Neural processes (NPs) share the PFN goal of learning to predict from a context set across a distribution of tasks. In particular, PFNs are closely related to transformer NPs, a transformer-based extension of conditional NPs \citep{nagler23statistical,nguyen2022tnp}. Conditional NPs directly parameterize predictive distributions~\citep{garnelo18a}, while latent NPs introduce a shared latent variable to model dependencies between predictions~\citep{garnelo2018neural}. Attentive neural processes use query-specific context representations~\citep{kim2018attentive}, and transformer neural processes formulate prediction through likelihood-based sequence modeling~\citep{nguyen2022tnp}. Closely related to our target-identification question, \citet{foong2020meta} show, under their stated conditions, that the population maximum-likelihood objective recovers the ground-truth stochastic process's prediction map. This supports interpreting learned predictions as approximations to conditional distributions, without guaranteeing exact Bayesian inference for a finite trained network.

\end{document}